\documentclass[twoside]{article}

\usepackage[accepted]{aistats2025}

\usepackage{microtype}
\usepackage{graphicx}
\usepackage{subfigure}
\usepackage{booktabs} 
\usepackage{comment}
\usepackage[T1]{fontenc}
\usepackage{lmodern} 

\usepackage{hyperref}

\usepackage{amsmath}
\usepackage{amssymb}
\usepackage{mathtools}
\usepackage{amsthm}

\usepackage{algorithm}
\usepackage{algorithmic}

\usepackage[capitalize,noabbrev]{cleveref}

\theoremstyle{plain}
\newtheorem{theorem}{Theorem}[section]
\newtheorem{proposition}[theorem]{Proposition}
\newtheorem{lemma}[theorem]{Lemma}

\theoremstyle{definition}
\newtheorem{definition}[theorem]{Definition}
\newtheorem{assumption}[theorem]{Assumption}
\theoremstyle{remark}

\usepackage[textsize=tiny]{todonotes}

\DeclareMathOperator*{\argmax}{arg\,max}
\DeclareMathOperator*{\argmin}{arg\,min}

\renewcommand{\eqref}[1]{(\ref{#1})}

\usepackage{enumitem}
\usepackage{bm}
\usepackage{bbm}
\usepackage{multirow}
\usepackage{natbib}
\allowdisplaybreaks

\newcommand{\annot}[2]{\underbrace{#1}_{\text{#2}}}

\usepackage{xcolor}
\newcount\Comments  
\newcommand{\kibitz}[2]{\ifnum\Comments=1\textcolor{#1}{#2}\fi}

\newcommand{\colorred}{ \color{black} }
\newcommand{\colorblack}{ \color{black} }

\usepackage{url}

\runningtitle{LC-Tsallis-INF: Generalized Best-of-Both-Worlds Linear Contextual Bandits}

\runningauthor{Masahiro Kato and Shinji Ito}

\begin{document}

\twocolumn[

\aistatstitle{LC-Tsallis-INF:\\
Generalized Best-of-Both-Worlds Linear Contextual Bandits}

\aistatsauthor{ Masahiro Kato$\,{}^{1, 2}$ \And Shinji Ito$\,{}^{2, 3}$ }

\aistatsaddress{${}^{1}\,$Mizuho-DL Financial Technology \And  ${}^{2}\,$The University of Tokyo \And ${}^{3}\,$RIKEN AIP} \aistatsaddress{masahiro-kato@fintec.co.jp \And shinji@mist.i.u-tokyo.ac.jp}]

\begin{abstract}
We investigate the \emph{linear contextual bandit problem} with independent and identically distributed (i.i.d.) contexts. In this problem, we aim to develop a \emph{Best-of-Both-Worlds} (BoBW) algorithm with regret upper bounds in both stochastic and adversarial regimes. We develop an algorithm based on \emph{Follow-The-Regularized-Leader} (FTRL) with Tsallis entropy, referred to as the $\alpha$-\emph{Linear-Contextual (LC)-Tsallis-INF}. We show that its regret is at most $O(\log(T))$ in the stochastic regime under the assumption that the suboptimality gap is uniformly bounded from below, and at most $O(\sqrt{T})$ in the adversarial regime. Furthermore, our regret analysis is extended to more general regimes characterized by the \emph{margin condition} with a parameter $\beta \in (1, \infty]$, which imposes a milder assumption on the suboptimality gap. We show that the proposed algorithm achieves $O\left(\log(T)^{\frac{1+\beta}{2+\beta}}T^{\frac{1}{2+\beta}}\right)$ regret under the margin condition.
\end{abstract}

\section{Introduction}
This study investigates the linear contextual bandit problem \citep{AbeLong1999}, which has been extensively studied across various domains, including sequential treatment allocation \citep{Tewari2017} and online advertising \citep{Li2010}. Depending on the domain, the behavior of loss and reward varies, and environment models have been developed to represent these differing behaviors. The primary models include the stochastic regime and the adversarial regime.

In the stochastic regime, losses are generated from a fixed distribution \citep{AbbasiYadkori2011}, whereas in the adversarial regime, losses are chosen to maximize the regret. For the stochastic regime, \citet{AbbasiYadkori2011} proposes OFUL. For the adversarial regime, \citet{Neu2020} proposes RealLinExp3. It is crucial to determine beforehand whether the environment is closer to a stochastic or adversarial setting to select an appropriate algorithm. However, it is often challenging to make this determination.

To address this difficulty, algorithms that perform well in both regimes, known as best-of-both-worlds (BoBW) algorithms~\citep{BubeckSlivkins2012}, have garnered attention. BoBW algorithms play a crucial role in real-world applications. For example, in online advertising, an advertisement algorithm is typically designed for a stochastic environment and deployed accordingly. Modeling the environment as stochastic corresponds to assuming that customer responses follow a fixed stationary probability distribution. However, customer responses can be influenced by the advertisements themselves. Some customers may respond less favorably after repeatedly seeing the same advertisements, thereby violating the stochastic assumption. In such cases, a purely stochastic algorithm may no longer perform well. Deploying an algorithm designed for an adversarial environment could mitigate this issue, as it is robust to worst-case scenarios. However, adversarial algorithms tend to be overly conservative and pessimistic, leading to suboptimal performance in environments that are neither fully stochastic nor fully adversarial. In such intermediate cases, a best-of-both-worlds algorithm proves effective, offering robustness while maintaining competitive performance.

For linear contextual bandits, \citet{kuroki2023bestofbothworlds} develops BoBW algorithms based on a black-box approach by \citet{Dann2023}, achieving $O\left(\log(T)\right)$ regret in stochastic regimes. However, these algorithms may not be practical due to implementation difficulties, as mentioned by \citet{kuroki2023bestofbothworlds}. To bypass this issue, \citet{kuroki2023bestofbothworlds} and \citet{kato2023bestofbothworlds} propose algorithms based on follow-the-regularized-leader (FTRL) with Shannon entropy regularization. These FTRL-based approaches are simple and computationally tractable, but their regrets remain suboptimal, at least $\log^2(T)$ for the number of rounds $T$.

Thus, for linear contextual bandits, the only BoBW algorithm with $O(\log(T))$ regret bounds relies on a complex black-box approach, and whether a simpler FTRL-based algorithm can achieve the same goal has remained an open question. This study resolves this question by providing an FTRL-based BoBW algorithm with regret bounds tightly dependent on $T$. To construct such algorithms, we employ FTRL with Tsallis entropy regularization. Furthermore, we consider a generalized setting for linear contextual bandits by introducing the margin condition \citep{Li2021regret}. For our algorithm, we derive regret upper bounds and show that in the stochastic regime, it is of order $O(\log(T))$.

\subsection{Problem setting}
We suppose that there are $T$ rounds and $K$ arms, denoted by $[T] \coloneqq \{1,2,\dots,T\}$ and $[K] \coloneqq \{1, 2, \dots, K\}$, respectively. In each round $t \in [T]$, each arm $a \in [K]$ is associated with a loss $\ell_t(a, X_t)$ given a context $X_t \in \mathcal{X}$, where $\mathcal{X}$ is an arbitrary context space. Here, $X_t$ is generated from a distribution $\mathcal{G}$ and induces a $d$-dimensional arm-dependent \emph{feature} vector $\phi(a, X_t) \in \mathcal{Z} \subset \mathbb{R}^d$ for each arm $a \in [K]$ (see Assumption~\ref{asm:contextual_dist}), where $\phi(a, \cdot)$ is a mapping from an arm-independent context to an arm-dependent feature.
The loss $\ell(a, x)$
follows linear models with $\phi(a, x)$.
\begin{assumption}[Linear models]
\label{asm:linearmodels}
For all $a \in [K]$ and any $x \in \mathcal{X}$, the following holds:
    \begin{align*}
    \ell_t(a, x) = \big\langle \phi(a, x), \theta_t \big\rangle + \varepsilon_{t}(a),
\end{align*}
where $\theta_t \in \Theta$ is a $d$-dimensional parameter with a space $\Theta \subset \mathbb{R}^d$, $\varepsilon_{t}(a)$ satisfies $\mathbb{E}\left[\varepsilon_{t}(a) \mid X_t, \mathcal{F}_{t-1}\right] = 0$, and $\mathcal{F}_{t-1} = \big\{\big(X_s, A_s, \ell_s(A_s, X_s)\big)\big\}^{t-1}_{s=1}$. 
\end{assumption}

The decision-maker follows a \emph{policy} in arm selection. Let $\Pi$ be the set of all possible policies $\pi:\mathcal{X}\to \mathcal{P}_K \coloneqq \Big\{ u = (u_1\ u_2\ \dots\ u_K)^\top \in [0, 1]^K \mid \sum^K_{k=1}u_k = 1 \Big\}$ with its $a$-th element $\pi(a \mid x)$. Then, we consider sequential decision-making with the following steps in each $t \in [T]$:
\begin{enumerate}[topsep=0pt, itemsep=0pt, partopsep=0pt, leftmargin=*]
\item Nature decides $\theta_t$ based on $\mathcal{F}_{t-1}$.
\item The decision-maker observes a context $X_t \in \mathcal{X}$, generated from a fixed distribution $\mathcal{G}$.
\item Based on the observed context $X_t$, the decision-maker selects a policy $\pi_t(X_t) \in \mathcal{P}_K$.
\item The decision-maker chooses an action $A_t \in [K]$ with probability $\pi_t(A_t \mid X_t)$.
\item The decision-maker incurs a loss $\ell_t(A_t, X_t)$.
\end{enumerate}
The goal of the decision-maker is to select actions to minimize the total loss $\sum^T_{t=1}\ell_t(A_t, X_t)$.

The performance is measured via the regret, defined as
\begin{align*}
    R_T &\coloneqq \max_{\rho \colon \mathcal{X} \rightarrow [K]} \mathbb{E}\left[\sum^T_{t=1} \big\{ \ell_t(A_t, X_t) - \ell_t(\rho(X_t), X_t) \big\} \right],
\end{align*}
where the expectation is taken over the randomness of policies and the contexts, $\{X_t\}_{t \in [T]}$, and losses, $\{\ell_t(\cdot, X_t)\}_{t \in [T]}$. The optimal policy $\rho^*_T$ is defined as
\begin{align*}
    \rho^*_T = \argmin_{\rho \colon \mathcal{X} \rightarrow [K]} \mathbb{E}\left[\sum^T_{t=1} \big\langle \phi(\rho(X_t), X_t), \theta_t \big\rangle \right].
\end{align*}
Thus, the regret becomes 
\[R_T = \mathbb{E}\left[\sum^T_{t=1} \big\langle \phi(A_t, X_t) - \phi(\rho^*_T(X_t), X_t), \theta_t \big\rangle \right].\]

\subsection{Contributions}
This study aims to propose a practical BoBW algorithm with $O(\log(T))$ regret in the stochastic regime without using the black-box framework. We focus on the use of Tsallis entropy instead of Shannon entropy. Section~\ref{sec:algorithm} presents our algorithm, the $\alpha$-Linear-Contextual-Tsallis-INF ($\alpha$-LC-Tsallis-INF), which is the FTRL with $\alpha$-Tsallis entropy regularization. Our algorithm's regret satisfies $O(\log(T))$. Additionally, compared to the algorithm using the black-box approach by \citet{kuroki2023bestofbothworlds}, our algorithm is easier to implement and has a tighter dependence regarding $K$.

A key component of our algorithm is the inverse of the covariance matrix of an arm-dependent feature, multiplied by a policy $\pi_t$ in each round $t \in [T]$, denoted as $\Sigma_t$ (Section~\ref{sec:assumption}). Throughout this study, we assume that $\Sigma_t$ can be computed exactly. When it is approximated using finite samples, the Matrix Geometric Resampling (MGR) method proposed by \citet{Neu2020} may be employed. Although we do not derive the regret bound when using the MGR, we provide a brief overview in Section~\ref{sec:algorithm}.

Our study also addresses general stochastic regimes depending on a margin condition with various parameters $\beta$. \citet{kato2023bestofbothworlds} and \citet{kuroki2023bestofbothworlds} only discuss a case where there exists a positive lower bound for the suboptimality gap, which is considered restrictive in linear contextual bandits. In contrast, we consider a milder assumption on the suboptimality gap, called a margin condition, and characterize the problem difficulty using a parameter $\beta \in (0, +\infty]$.

We then derive regret upper bounds of the $1/2$-LC-Tsallis-INF ($\alpha$-LC-Tsallis-INF with $\alpha = 1/2$) for each regime defined in Section~\ref{sec:regretanalysis}. In an adversarial regime, the regret is given as 
$ O\left(\sqrt{d \sqrt{K} T} + \frac{1}{\lambda}\left(\sqrt{K} + \frac{\sqrt{\sqrt{K}T}}{\sqrt{d}}\right) \right),$ where  $\lambda$ is a parameter that characterizes $\Sigma^{-1}$ (Assumption~\ref{asm:contextual_dist2}) and a generalization of $\lambda_{\min}$ in the existing studies. In a stochastic regime with a margin condition, the regret is given as 
$O\left(
    \left(\frac{1 + \beta}{\beta\Delta_*}\right)^{\frac{\beta}{2+\beta}}\left(L d \sqrt{K} \log(T)\right)^{\frac{1+\beta}{2+\beta}}T^{\frac{1}{2+\beta}}
    +
    \kappa
\right)$, where $\kappa =
    O\left(
        \sqrt{\frac{KL}{\lambda}}
        +
        \frac{\sqrt{K}L}{\lambda d} \log (T)
    \right)$, and $L$ is the lowest probability of contexts (see Assumption~\ref{asm:contextual_dist_stochastic}). In an adversarial regime with a self-bounding constraint, the regret is given as 
    $O\left(\frac{L d \sqrt{K}}{\Delta_*}\log(T) + \sqrt{\frac{C  L\sqrt{K}d }{\Delta_*}\log_+\left(\frac{\Delta_*T}{CL} \right)} + \kappa
\right)$, where $\log_+(x) \coloneqq \max\{1, \log(x)\}$ and recall that $\beta$ is a parameter of a margin condition.

The advantage of using Tsallis entropy lies in its improved regret bounds, particularly in the dependency on $T$ in the stochastic regime. When using Shannon entropy, the regret does not achieve the optimal $\log(T)$ dependency. For example, \citet{kuroki2023bestofbothworlds} develops a best-of-both-worlds algorithm with a regret bound of $\log(T)^2$ using Shannon entropy.

We also discuss the difference between settings with arm-dependent and arm-independent features in Section~\ref{sec:regret_transform}. While the former is the setting employed in this study and most existing studies in linear contextual bandits, the latter is a setting used in \citet{Neu2020}, \citet{kuroki2023bestofbothworlds}, and \citet{kato2023bestofbothworlds}. We point out that if we can derive a regret in either of the settings, we can obtain a regret in the other. We highlight that a tighter regret can be achieved if we first derive a regret under a setting with arm-dependent features and then transform it to a setting with arm-independent features. In Table~\ref{table:comparison}, we report the upper bound for RealLinExp3 from \citet{Neu2020} after applying the regret transformation described in Section~\ref{sec:regret_transform}. By Theorem~2 of \citet{Neu2020}, their original regret bound is
$O\left(\sqrt{TK \max\{d,\frac{\log T}{\lambda}\}\log K}\right)$. 
Because their setting is based on arm-independent features (differing from our setup with arm-dependent features), we applied the transformation. By applying Theorem~\ref{thm:reg_trans}, the parameter \(d\) is replaced with \(dK\). Consequently, we obtain the regret upper bound shown in Table~\ref{table:comparison}.

In summary, our contributions lie in the proposition of an FTRL-based algorithm whose upper bound tightly depends on $T$, and the analysis under a margin condition and the arm-dependent feature setting. In Table~\ref{table:comparison}, we compare our algorithm's regrets with existing ones.

\begin{table*}[t]
 \caption{Comparison of regrets. We categorize regrets based on regimes. We place a $\checkmark$ in the ``$\sqrt{C}$'' column if the regret depends on the corruption level $C \geq 0$ in the presence of adversarial corruption. We transformed the regret from the arm-independent setting to the arm-dependent setting using Theorem~\ref{thm:reg_trans}. For the BoBW with Shannon entropy regularization, we also present a regret derived for the arm-dependent setting by ourselves.
}
    \label{table:comparison}
    \centering
    \scalebox{0.85}[0.85]{
    \begin{tabular}{|c|c|c|c|}
    \hline
    & \multicolumn{2}{|c|}{Regret} &   \multirow{2}{*}{$\sqrt{C}$} \\
    \cline{2-3}
        & Stochastic  & Adversarial &  \\
    \hline
        &  \multirow{2}{*}{$O\left(\frac{1}{\Delta_*}L d \sqrt{K}\log(T) + \kappa
    \right)$}& \multirow{5}{*}{$O\left(\sqrt{d \sqrt{K} T} + \kappa \right)$} & \multirow{2}{*}{\checkmark} \\
         $1/2$-LC-Tsallis-INF & & & \\
        Section~\ref{sec:eachregime} & $O\left(
        \left(\frac{1 + \beta}{\beta\Delta_*}\right)^{\frac{\beta}{2+\beta}}\left(L d \sqrt{K} \log(T)\right)^{\frac{1+\beta}{2+\beta}}T^{\frac{1}{2+\beta}}
        +
        \kappa
    \right)$ &  & - \\
        & under a margin condition with  $\beta \in (0, \infty]$. & & \\

        & \multicolumn{2}{|c|}{$\kappa =
        O\left(
            \sqrt{\frac{KL}{\lambda}}
            +
            \frac{\sqrt{K}L}{\lambda d} \log (T)
        \right)$} & \\
        \hline
        FTRL with Shannon entropy &  \multirow{2}{*}{$O\left(\frac{1}{\Delta_{*}}K
        \left(dK + \frac{\log(T)}{\lambda_{\min}}\right)\log\left(KT\right)\log(T)\right)$} & \multirow{2}{*}{$O\left(\sqrt{\log(T)\log(K)TK\left(dK + \frac{\log(T)}{\lambda_{\min}}\right)}\right)$} & \checkmark \\
        \citep{kuroki2023bestofbothworlds} &  & &  \\
        \citep{kato2023bestofbothworlds} & \multicolumn{2}{|c|}{under transformation using our Thm~4.5.} &  \\
        \cline{2-3}
        Section~\ref{sec:regret_transform} & \multirow{2}{*}{$O\left(\frac{1}{\Delta_{*}}\left(d + \frac{1}{ \lambda}\right)\log(KT)\log(T)\right)$} & \multirow{2}{*}{${O}\left(\sqrt{ \log(T)T\left(d + \frac{1}{ \lambda}\right)}\log(KT)\right)$} &  \\
        and Appendix~\ref{sec:bobw_linearftrl} &  & &  \\
         & \multicolumn{2}{|c|}{from Appendix~D.} &  \\
     \hline
     BoBW reduction & \multirow{3}{*}{$O\left(\frac{1}{\Delta_*}K^2\left(dK + \frac{1}   {\lambda_{\min}}\right)^2\log(K)\log(T)\right)$} & \multirow{3}{*}{$O\left(\sqrt{TK^2\left(dK + \frac{1}{\lambda_{\min}}\right)^2\log(K)}\right)$} & \multirow{3}{*}{\checkmark} \\
     of the RealLinExp3 & & & \\
     \citep[Prop.~8,][]{kuroki2023bestofbothworlds} & & & \\
     \hline
      MWU-LC & \multirow{2}{*}{$O\left(\frac{1}{\Delta_*}d^2K^4\log^2(dK^2T)\log^3(T)\right)$} & \multirow{2}{*}{$O\left(dK^2\sqrt{\Lambda^*}\log(T)\right)$} & \multirow{2}{*}{\checkmark} \\
     \citep[Thm.~1,][]{kuroki2023bestofbothworlds} & & & \\
     \hline
     Logdet-FTRL & \multirow{2}{*}{-} & \multirow{2}{*}{$O\left(d^2\sqrt{T}\log(T)\right)$}  & \multirow{2}{*}{-} \\
     ($\mathrm{poly}(K, d, T)$ in computation) & & &  \\
    Linear EXP4 & \multirow{2}{*}{-} & \multirow{2}{*}{$O\left(d\sqrt{T\log(T)}\right)$}  & \multirow{2}{*}{-} \\
    ($T^d$ in computation) & & & \\
    \citep{liu2023bypassing} & & & \\
 \hline
        OFUL  & $O\left(d\log(1/\delta) / \Delta_* \right)$  with probability $1-\delta$ & - & -  \\
     \citep{AbbasiYadkori2011} & $O\left(d\log(T) / \Delta_* \right)$ with probability $1 - 1/T$ & - & -  \\    
    \hline
    RealLinExp3 & \multirow{2}{*}{-} & \multirow{2}{*}{$O\left(\sqrt{TK\max\left\{dK, \frac{\log(T)}{\lambda_{\min}}\right\}\log(K)}\right)$}  & \multirow{2}{*}{-} \\
    \citep{Neu2020} & & & \\
    \hline
    \end{tabular}
    }
    \vspace{-5mm}
\end{table*}

\subsection{Related work}
For linear contextual bandits, \citet{kuroki2023bestofbothworlds} and \citet{kato2023bestofbothworlds} propose FTRL-based BoBW algorithms. \citet{kuroki2023bestofbothworlds} develop several BoBW algorithms using the black-box framework approach by \citet{Dann2023}. Among them, they propose a BoBW reduction of RealEXP3 by \citet{Neu2020}, showing a regret of order $O(\log(T))$ in the stochastic regime. They also present another algorithm by combining the black-box framework with the continuous exponential weights algorithm investigated by \citet{olkhovskaya2023first}, which has a regret of $O(\log^5(T))$ in the stochastic regime and $O(\log(T) d K \sqrt{\Lambda^*})$ in the adversarial regime, where $\Lambda^*$ denotes the cumulative second moment of the losses incurred by the algorithm.

While the black-box framework provides a tight $O(\log(T))$ regret regarding $T$, limitations have been reported. For example, \citet{kuroki2023bestofbothworlds} mention that ``it may not be practical to implement.'' Compared to the black-box framework, ``FTRL with Shannon entropy regularization is a much more practical algorithm'' \citep{kuroki2023bestofbothworlds}. \citet{kuroki2023bestofbothworlds} and \citet{kato2023bestofbothworlds} show that the algorithm has a regret of $O\left(\frac{K}{\Delta_{*}}\left(dK + \frac{\log(T)}{\lambda_{\min}}\right)\log\left(KT\right)\log(T)\right)$ in the linear contextual adversarial regime with a self-bounding constraint, where $\lambda_{\min}$ is the smallest eigenvalue of a feature covariance matrix induced by an exploratory policy (see Assumption~\ref{asm:contextual_dist2}). Note that we transformed their original regrets derived for the arm-independent feature setting to the ones for the arm-dependent feature setting by using our Theorem~\ref{thm:reg_trans} in Section~\ref{sec:regret_transform}. They also show a regret of $O\left(\sqrt{TK \left(dK + \frac{\log(T)}{\lambda_{\min}}\right)} \log(T) \log(K)\right)$ in the adversarial regime. Note that their FTRL with Shannon entropy considers a case where $\Sigma_t^{-1}$, the inverse of the feature covariance matrix, is approximated by finite samples, and they assume that features are arm-independent (see Section~\ref{sec:regret_transform}). Here, the approximation error of $\Sigma_t^{-1}$ affects the regrets. If $\Sigma_t^{-1}$ can be exactly computed, we can remove $\log(T)$ in $\log(T)/\lambda_{\min}$, and the regret in the stochastic regime becomes of order $O(\log^2(T))$ regarding $T$ instead of $\log^3(T)$.

BoBW algorithms based on FTRL with Shannon entropy usually incur regret of at least $O(\log^2(T))$~\citep{ito2022nearly,tsuchiya2023best,kong2023best}. As a potential solution to this issue, using Tsallis entropy instead of Shannon entropy has been shown to be effective, achieving success in multi-armed bandits~\citep{Zimmert2019,Masoudian2021,jin2023improved}, combinatorial semi-bandits~\citep{zimmert2019beating}, dueling bandits~\citep{saha2022versatile}, and graph bandits~\citep{rouyer2022near}. However, for linear contextual bandits, even when restricted to the adversarial setting, algorithms using Tsallis entropy are not known, and the application and analysis of this approach are not straightforward.





\textbf{Notation} Let $\langle \cdot, \cdot \rangle$ denote inner products in Euclidean space and let $\|\cdot\|_2$ denote the $\ell_2$ norm.

\section{Preliminaries}
\label{sec:prob}
This section provides assumptions on our problem. This study considers a setting where the arm-dependent feature is given, and the parameters are arm-independent.

\subsection{Boundedness of variables}
\label{sec:assumption}
We first assume the boundedness of the variables.
\begin{assumption}[Bounded loss]
    \label{asm:bounded_random}
    We assume that $| \left\langle z, \theta \right\rangle| \le 1$ for all $z \in \mathcal{Z}$ and $\theta \in \Theta$.
    In addition,
    we assume that $|\ell_t( a, x )| =
    \left|\left\langle \phi(a, x), \theta_t \right\rangle + \varepsilon_{t}(a)\right| \le 1$
    holds for all $a \in [K]$, $x \in \mathcal{X}$ and $t \in [T]$.
\end{assumption}
The parameter $\theta_t$ is generated in different ways according to the data-generating process (DGP). We define regimes of the DGP in Section~\ref{sec:dgp}. 

\subsection{Assumptions on contexts}

\colorred

\paragraph{Contexts and feature map}
We first assume that an i.i.d.~random variable $X_t$ is generated from a fixed distribution $\mathcal{G}$ over the support $\mathcal{X}$. We refer to $X_t$ as contexts. Features are obtained from $X_t$ through a feature map that transforms a context $x \in \mathcal{X}$ into a feature $\phi(a, x) \in \mathcal{Z} \subseteq \mathbb{R}^d$, where $\mathcal{Z}$ is a $d$-dimensional feature space. We assume that $\phi$ is known.

For any $p = \{ p_x \}_{x \in \mathcal{X}}$, a set of conditional distributions $p_x \in \mathcal{P}_K$ given $x \in \mathcal{X}$, we define a matrix $\Sigma(p) \in \mathbb{R}^{d \times d}$ by $\Sigma(p) \coloneq \mathbb{E}_{X_0 \sim \mathcal{G}, a \sim p_{x}} \Big[\phi(a, X_0)\left(\phi(a, X_0)\right)^\top\Big]$, where $X_0$ is a sample from $\mathcal{G}$. This matrix plays an important role in constructing unbiased estimators of loss and in the analysis of regret. We assume that $\Sigma^{-1}(p)$ can be computed exactly. Below, we summarize the assumptions about the contexts.

\begin{assumption}[Contextual distribution]
\label{asm:contextual_dist}
(i) Context $X_t \in \mathcal{X}$ is an i.i.d. random variable from a contextual distribution $\mathcal{G}$ with support $\mathcal{X}$. (ii) There is a \emph{known} feature map $\phi:[K]\times\mathcal{X}\to \mathcal{Z}$, which maps $x\in\mathcal{X}$ to feature $\phi(a, x) \in \mathcal{Z}\subset \mathbb{R}^d$. (iii) For any $p\in\Pi$, $\Sigma^{-1}(p)$ is exactly computable.
\end{assumption}

Conditions (i) and (ii) are standard assumptions in adversarial linear contextual bandits \citep{Neu2020} and are essential for the algorithm design. 

Under condition~(i), the regret analysis can be reduced to evaluating the pointwise regret for each \(x\). Similar analytical techniques have also been used in previous studies, such as \citet{Neu2020}.

Although condition (iii) may appear restrictive, it is not necessarily stronger than assumptions in existing studies, such as \citet{Neu2020}, \citet{kuroki2023bestofbothworlds}, and \citet{kato2023bestofbothworlds}, which approximate $\Sigma^{-1}(p)$ using finite samples from a known $\mathcal{G}$ and the MGR algorithm. For instance, if the support is finite and its probability mass function ($g(x)$ in Assumption~\ref{asm:contextual_dist_stochastic}) is known, we can compute the exact $\Sigma^{-1}(p)$. If there are infinite samples from $\mathcal{G}$, then we can approximate $g(x)$ using these samples.

Moreover, the MGR method is computationally inefficient in the BoBW algorithms proposed by \citet{kuroki2023bestofbothworlds} and \citet{kato2023bestofbothworlds}. While the MGR is computationally efficient in adversarial linear contextual bandits \citep{Neu2020}, the computational costs are not bounded in those BoBW algorithms. We discuss this point further in Section~\ref{sec:approx_sigma}.

\paragraph{Contexts in a stochastic regime}
In a stochastic regime with adversarial corruption, we additionally make the following assumption on contexts.
\begin{assumption}[Finite support]
\label{asm:contextual_dist_stochastic}
(i) The context space $\mathcal{X}$ is 
a finite set of size $S = |\mathcal{X}| \in \mathbb{N}$.
(ii) There exists a constant $L \geq S$ such that $g(x) := \mathbb{P}_{X_0 \sim \mathcal{G}}\big(X_0 = x\big) \geq 1/L$ holds for all $x \in \mathcal{X}$.
\end{assumption}
This assumption is reasonable in some applications. For instance, in advertising, arm selections are often based on contexts provided by personal information, such as age or gender. When considering regret minimization for a specific group (e.g., people in the US), we can often obtain its distribution from publicly available data, such as demographics. Moreover, when contexts are continuous, practitioners often discretize them and select arms based on groups for ease of implementation. 

\colorblack

\paragraph{Exploration policy}
We use an \textit{exploration policy} that satisfies Assumption~\ref{asm:contextual_dist2} in the algorithm design.
\begin{assumption}[Exploration policy]
\label{asm:contextual_dist2}
    There exists a constant $\lambda > 0$ and a set $e^* = \{ e^*_x \}_{x \in \mathcal{X}}$ of distributions over $[K]$
    such that $z^\top \Sigma(e^*)^{-1} z \le 1 / \lambda$ for all $z \in \mathcal{Z}$.
\end{assumption}
The value of $\lambda$ in this assumption can be interpreted as a generalized or a relaxed version of the smallest eigenvalue $\lambda_{\min}$ of $\Sigma(e^*)$,
which is used in assumptions of existing studies \citep{Neu2020,kuroki2023bestofbothworlds}.
In fact,
if $\| z \|_2 \le 1$ holds for all $z \in \mathcal{Z}$,
$\max_{z \in \mathcal{Z}} z^{\top} \left( \Sigma(e^*) \right)^{-1} z$ is at most $1 / \lambda_{\min}$.
In addition,
if we choose $e^*_x$ to be a G-optimal design (see, e.g., Chapter~21 by \citet{lattimore2020bandit}) for $\mathcal{Z}_x = \{ \phi(a, x) \in \mathcal{Z} \mid a \in [K] \}$,
we may set $\lambda = 1 / (dL)$.
In fact,
it holds that $
    \phi(a,x)^\top \left( \Sigma( e^* ) \right)^{-1} \phi(a,x)
    =
    \phi(a,x)^\top \left( \sum_{x' \in \mathcal{X}} g(x') \Sigma( e^*_{x'} ) \right)^{-1} \phi(a,x)
    \le
    \phi(a,x)^\top$ $\left(  g(x) \Sigma( e^*_{x} ) \right)^{-1}\phi(a,x)
    \le
    d / g(x)
    \le
    dL$ for all $a \in [K]$ and $x \in \mathcal{X}$. As discussed in Note~6 of Section~21.2 in \citet{lattimore2020bandit}, we can define the G-optimal design even when $\mathcal{Z}$ does not span the entire space. For example, we can use Moore-Penrose pseudoinverses in place of inverse matrices.


\subsection{DGP: stochastic and adversarial regimes}
\label{sec:dgp}
We define three regimes for the DGP of $\big\{\theta_t\big\}_{t\in[T]}$: an adversarial regime, a stochastic regime with a margin condition, and an adversarial regime with a self-bounding constraint.

\paragraph{(1) Adversarial regime}
First, we introduce the adversarial regime, where we do not make any assumptions about the behavior of the nature. In this case, it is known that the lower bound is $O(\sqrt{T})$ when there is no context \citep{Auer2002}. 

Note that adversarial linear contextual bandits can be defined in various ways. For example, some studies consider adversarial contexts and fixed losses \citep{Chu2011contextual, AbbasiYadkori2011}. Meanwhile, other studies address adversarial contexts and adversarial losses \citep{KanadeSteinke2014, HazanKoren2016}. This study focuses exclusively on contextual bandits with i.i.d. contexts and adversarial losses, as studied by \citet{Rakhlin2016} and \citet{Syrgkanis2016improved}. This study follows the setting of \citet{Neu2020}, \citet{kato2023bestofbothworlds}, and \citet{kuroki2023bestofbothworlds}.

\paragraph{(2) Stochastic regime with a margin condition}
Next, we define a \emph{margin condition}, which is often assumed in linear contextual bandits to characterize the difficulty of the problem instance \citep{Li2021regret}. Since this section focuses on a stochastic regime, regression coefficients are fixed, and we denote them by $\theta_0$; that is, $\theta_1 =\cdots = \theta_t = \theta_0$. 
Note that under a stochastic regime, it holds that
\begin{align}
    R_T  &= \mathbb{E}\left[\sum^T_{t=1}\sum_{a\in[K]}\Delta(a \mid X_t) \pi_t(a\mid X_t)\right]
    \label{eq:reg_lower}\\
    &\geq \mathbb{E}\left[\sum^T_{t=1}\Delta(X_t) \big(1 - \pi_t(\rho^*_T(X_t)\mid X_t)\big)\right],\nonumber
\end{align}
where 
\begin{align*}
  &\Delta(a\mid x) = \Big\langle \phi(a, x) - \phi(\rho^*_T(x), x), \theta_{0}\Big\rangle\ \ \ \mathrm{and}\\
  &\Delta(x) = \min_{a\neq \rho^*_T(x)}\Bigg\{\Big\langle \phi(a, x), \theta_0\Big\rangle - \Big\langle \phi(\rho^*_T(x), x), \theta_0\Big\rangle\Bigg\}.  
\end{align*}

Based on this suboptimality gap $\Delta(x)$, we define a stochastic regime with a margin condition.
\begin{definition}[Stochastic regime with a margin condition]
\label{def:margin}
Consider the stochastic regime with fixed regression coefficient $\theta_0 \in \Theta$, where for all $t\in[T]$, all $a\in[K]$ and any $x\in\mathcal{X}$, the loss is generated as $\ell_t(a, x) = \big\langle \phi(a, x), \theta_0 \big\rangle + \varepsilon_{t}(a)$. Furthermore, there exists a universal constant $\Delta_* > 0$ and $\beta \in (0, +\infty]$, such that for any $h\in\left[0, \Delta_*\right]$, it holds that
\[\mathbb{P}\left(\Delta(X_t) \leq h\right)\leq \frac{1}{2}\left(\frac{h}{\Delta_*}\right)^\beta.\]
\end{definition}

Note that when $\beta = \infty$, $\Delta(x) \geq \Delta_*$ holds for any $x\in\mathcal{X}$ and $t\in[T]$. A margin condition in \citet{Li2021regret} restricts the range of $h$ as $[D\sqrt{\log(d)/T}, \Delta_*]$ for some universal constant $D > 0$, and they derive matching lower and upper bounds.

\colorred

The margin condition is one of the minimal assumptions for achieving $o(\sqrt{T})$-regret in linear contextual bandits. For instance, the lower bound in Theorem~1 by \citet{Li2021regret} suggests that the achievable regret bounds can be characterized by the margin parameter. If the condition does not hold, the regret becomes $\Omega(\sqrt{T})$. 

\colorblack

\paragraph{(3) Adversarial regime with a self-bounding constraint}
This section defines a regime with an adversarial corruption. Let $\Delta_* > 0$ be a universal constant, as used in Definition~\ref{def:margin}. If $\Delta(x) \geq \Delta_*$ holds for any $x\in\mathcal{X}$ in a stochastic regime, then the regret can be lower bounded as 
$R_T \geq \Delta_*\mathbb{E}\left[\sum^T_{t=1}\left( 1 - \pi_t\left(\rho^*_T(X_t) \right) \right)\right]$.
Based on this intuition,
we define an \emph{adversarial regime with a self-bounding constraint} below, as well as introduced by \citet{kuroki2023bestofbothworlds} and \citet{kato2023bestofbothworlds}. 
\begin{definition}[Adversarial regime with a self-bounding constraint]
\label{def:bounding}
We say that the DGP is in a $(\Delta_*, C, T)$-adversarial regime with a self-bounding constraint for some $\Delta_*, C > 0$ if the regret $R_T$ is lower bounded as
\begin{align*}
R_T \geq \mathbb{E}\left[\sum^T_{t=1}\sum_{a\neq \rho^*_T(X_t)}\Delta_t(a\mid X_t) \pi_t(a\mid X_t)\right] - C,
\end{align*}
where $\Delta_t(a\mid x) = \Big\langle \phi(a, x), \theta_t\Big\rangle - \Big\langle \phi(\rho^*_T(x), x), \theta_t\Big\rangle$, and for 
any $x\in\mathcal{X}$ and $a\in[K] \setminus \{ \rho^*_T(x) \}$,
$\Delta_t(a\mid x) \geq \Delta_*$ holds. 
\end{definition}

An adversarial regime with a self-bounding constraint encompasses several important settings. See examples in \citet{kato2023bestofbothworlds}. Note that in an adversarial regime, there may exist $a \in [K]$ and  $x\in\mathcal{X}$ such that $\Delta_t(x) \coloneqq \Big\langle \phi(a, x), \theta_t\Big\rangle - \Big\langle \phi(\rho^*_T(x), x), \theta_t\Big\rangle < 0$. This is because $\Delta(a\mid x)$ and $\Delta_{s}(a\mid x)$ can take a different value for some $a \in [K]$ and $x\in\mathcal{X}$ if $t \neq s$.

\section{Algorithm: \texorpdfstring{$\alpha$}{TEXT}-LC-Tsallis-INF}
\label{sec:algorithm}
This section provides an algorithm for linear contextual bandits with adversarial corruption. We refer to our algorithm as the $\alpha$-Linear-Contextual (LC)-Tsallis-INF because it modifies the Tsallis-INF \citep{Zimmert2019}, an FTRL-based algorithm with Tsallis entropy regularization without contexts. Here, $\alpha \in (0, 1)$ is a parameter of Tsallis-entropy. While our algorithm is defined for general $\alpha$, we show the regret bounds only for $\alpha = 1/2$. The pseudo-code is shown in Algorithm~\ref{alg}.


When selecting arm $a$ with probability $\pi_t(a\mid x)$ given $x\in\mathcal{X}$, we denote $\Sigma((\pi_t(x))_{x\in\mathcal{X}})$ by $\Sigma_t$, equal to
\[\Sigma_t = \mathbb{E}_{X_t\sim \mathcal{G}}\left[\sum_{a\in[K]}\pi_t(a\mid X_t)\phi(a, X_t)\left(\phi(a, X_t)\right)^\top\right].\] 

\subsection{Regression coefficient estimator}
We define an estimator of the parameter $\theta_t$ as
\begin{align}
\label{eq:defthetahat}
    \widehat{\theta}_{t} \coloneqq \widehat{\theta}_{t}\left( \Sigma^{-1}_t\right) \coloneqq \Sigma^{-1}_t\phi\big(A_t, X_{t}\big)\ell_{t}(A_t, X_t).
\end{align}
This estimator is unbiased for $\theta_t$ since it holds that
\begin{align*}
    \mathbb{E}\left[\widehat{\theta}_{t}(\Sigma^{-1}_t) \mid \mathcal{F}_{t-1}\right] 
    &= \Sigma^{-1}_t\Sigma_t\theta_t = \theta_t. 
\end{align*}
Here, recall that $\Sigma_t$ is defined as an expectation taken w.r.t. $X_0$ and $a$ given $\pi_t$ (as defined in Section 2.2).

In each round $t$, using $\widehat{\theta}_{t}$, we construct an estimator of the loss $\ell_t(x)$ as $\widehat{\ell}_t(x) \coloneqq \left(\langle{\widehat{\theta}_{t}, \phi\big(a, X_{t}\big)}\right)_{a\in [K]}$. This estimator is unbiased for the true loss $\ell_t(x)$. 

\subsection{FTRL with Tsallis entropy regularization}
By using the estimator $\widehat{\theta}_{t}\left(\Sigma^{-1}_t\right)$, we define the $\alpha$-LC-Tsallis-INF. In each round $t \in [T]$, the $\alpha$-LC-Tsallis-INF selects an arm with the following policy:
\begin{align}
    \label{eq:policy}
    \pi_t(X_t) \coloneqq (1-\gamma_t)q_t(X_t) + \gamma_te^*(X_t),
\end{align}
where recall that $e^*$ is an exploration policy defined in Assumption~\ref{asm:contextual_dist2}, and
\begin{align}
    \label{eq:q_opt}
    &q_t(x) \coloneqq \argmin_{q\in \mathcal{P}_K}\left\{ \sum^{t-1}_{s=1}\left\langle \widehat{\ell}_s(x), q\right\rangle + \frac{1}{\eta_{t-1}}\psi\big(q\big)\right\},
    \\
    &\psi(q(x)) \coloneqq \frac{1}{\alpha}\left( 1 - \sum_{a\in[K]}q(a\mid x)^\alpha\right),\nonumber\\
    &
    \widetilde{\eta}_t \coloneqq \frac{K^{1/4}}{\sqrt{ d t} },
    \quad \gamma_t \coloneqq \frac{128 L \eta_t^2}{\lambda} \left( \le \frac{1}{2} \right),\nonumber\\
    &\eta_t \coloneqq \begin{cases}
       \min \left\{ \widetilde{\eta}_t,~\frac{1}{16}\sqrt{\frac{\lambda}{L}} \right\} & \mathrm{if\ \ Assumption~\ref{asm:contextual_dist_stochastic}\ \  holds}\\
       \min \left\{ \widetilde{\eta}_t,~\frac{\lambda}{16} \right\} & \mathrm{otherwise}
    \end{cases}.\nonumber
\end{align}

Here, the regularizer $\psi(q(x))$ is referred to as the $\alpha$-Tsallis entropy \citep{tsallis1988possible}. 

\subsection{Approximation of \texorpdfstring{$\Sigma^{-1}_t$}{TEXT}}
\label{sec:approx_sigma}
\colorred
If $\Sigma^{-1}_t$ is not exactly computable, we can use the MGR to approximate it using finite samples from $\mathcal{G}$ (Section~\ref{sec:assumption}). The MGR is a computationally efficient algorithm for approximating $\Sigma^{-1}(p)$ in the adversarial regime, as demonstrated by \citet{Neu2020}. 

However, in the BoBW setting, the MGR incurs a significantly higher computational cost \citep{kuroki2023bestofbothworlds,kato2023bestofbothworlds}. Approximating $\Sigma^{-1}(p)$ using the MGR from finite samples drawn from $\mathcal{G}$ results in a computational cost of order $O\left(d^2 K M^2_t\right)$, where $M^2_t$ is a term that increases at least linearly with $t$. While the computational burden of the MGR is manageable in the adversarial regime \citep{Neu2020}, it becomes substantial when aiming for a BoBW regret guarantee. For example, if we use the MGR in the FTRL with Shannon entropy regularization proposed by \citet{kuroki2023bestofbothworlds} and \citet{kato2023bestofbothworlds}, $M_t$ can be $\Omega(\sqrt{t})$, which results in a computational cost of at least order $t$ with respect to $t$.

In contrast, if contexts have finite support with bounded probabilities (Assumption~\ref{asm:contextual_dist_stochastic}), it is possible to compute $\Sigma^{-1}(p)$ exactly with an acceptable computational cost, as assumed in this study. Specifically, when the probability is lower bounded by $L$, $\Sigma^{-1}(p)$ can be computed with a cost of order $O(d^2 K |\mathcal{S}|)$. 

Thus, if our primary interest lies in achieving BoBW guarantees, the MGR may not be the most computationally efficient choice. Compared to the MGR, our approach, which relies on Assumption~\ref{asm:contextual_dist_stochastic}, offers a more computationally feasible alternative.

\colorblack

\begin{algorithm}[tb]
    \caption{$\alpha$-LC-Tsallis-INF.}
    \label{alg}
    \begin{algorithmic}
    \STATE {\bfseries Parameter:} Learning rate $\eta_1, \eta_2,\dots, \eta_T > 0$. 
    \FOR{$t=1,\dots, T$}
    \STATE Observe $X_t$. 
    \STATE Draw $A_t \in [K]$ following the policy $\pi_t(X_t) \coloneqq (1-\gamma_t)q_t(X_t) + \gamma_te^*(X_t)$ defined in \eqref{eq:policy}.
    \STATE Observe the loss $\ell_t(A_t, X_t)$.
    \STATE Compute $\widehat{\theta}_t$. 
    \ENDFOR
\end{algorithmic}
\end{algorithm}

\section{Regret analysis for the \texorpdfstring{$1/2$}{TEXT}-LC-Tsallis-INF}
\label{sec:regretanalysis}
This section provides regrets of the $\alpha$-LC-Tsallis-INF with $\alpha = 1/2$ ($1/2$-LC-Tsallis-INF).

First, we show the following upper bound,
from which we derive upper bounds in adversarial and stochastic regimes.
We prove Theorem~\ref{thm:regret_bound} in Appendix~\ref{appdx:proof_main}. 
\begin{theorem}[General regret bounds]
    \label{thm:regret_bound}
    Consider the $1/2$-LC-Tsallis-INF. 
    Assumptions~\ref{asm:contextual_dist} and \ref{asm:contextual_dist2}--\ref{asm:bounded_random} hold. 
    Then, the regret satisfies 
    \begin{align*}
        &R_T = O\left(
            \mathbb{E}\left[\sum^T_{t=1}\frac{\sqrt{ d \sqrt{K} \omega_t}}{\sqrt{t}}\right]
            +
            \widetilde{\kappa}
        \right),\quad \mathrm{where}\ \widetilde{\kappa} \coloneqq \\
        &
        \begin{cases}
       O\left(
            \sqrt{\frac{KL}{\lambda}}
            +
            \frac{\sqrt{K}L\log (T)}{\lambda d} 
        \right) & \mathrm{if\ Assumption~\ref{asm:contextual_dist_stochastic}\ holds}\\
       O\left(\frac{1}{\lambda}\left(\sqrt{K}
        +
        \frac{\sqrt{\sqrt{K}T}}{\sqrt{d}}
    \right) \right) & \mathrm{otherwise}
    \end{cases},
    \end{align*}
and $\omega_t \in [0, 1]$ is given as
    \[
        \omega_t \coloneqq \max_{x\in\mathcal{X}}\min_{a\in[K]}\left\{ 1 - q_t(a\mid x) \right\}.\]
\end{theorem}

Using this general regret bound, we derive a corresponding bound under an adversarial regime in Theorem~\ref{thm:regret_bound2}. In Theorem~\ref{thm:regret_upper_margin}, we present a regret bound under a stochastic regime with a margin condition. Lastly, in Theorem~\ref{thm:regret_bound3}, we develop a regret bound under a stochastic regime that achieves an order of \(\log(T)\).

In these derivations, the term \(\omega_t\) plays a pivotal role, corresponding to an upper bound of the probabilities of choosing suboptimal arms. For the adversarial regime, the regret is derived without placing any restrictions on \(\omega_t\), whereas for the stochastic regime, we leverage specific properties of \(\omega_t\). In particular, we employ a self-bounding technique, which tightens the regret bound by exploiting the coherence between the regret and \(\omega_t\) in the stochastic regime.

\subsection{Regret upper bound in each regime}
\label{sec:eachregime}
Because $\omega_t \leq 1$, by replacing $\omega_t$ with $1$ in Theorem~\ref{thm:regret_bound}, we can directly obtain a regret upper bound in the adversarial regime in the following theorem. 

We show the proof in Appendix~\ref{sec:proof1} and \ref{sec:proof2}.
As well as \citet{kato2023bestofbothworlds} and \citet{kuroki2023bestofbothworlds}, we employ the self-bounding technique~\citep{Zimmert2019,wei2018more,Masoudian2021} and an entropy-adaptive update rule for learning rates, which have been proven effective in providing BoBW guarantees for online learning in feedback graph contexts~\citep{ito2022nearly}, multi-armed bandits \citep{jin2023improved}, partial monitoring~\citep{tsuchiya2023best}, linear bandits~\citep{kong2023best}, episodic Markov Decision Processes (MDPs)~\citep{dann2023best},
and sparse bandits~\citep{tsuchiya2023stability}.

\begin{theorem}[Regret upper bound in an adversarial regime]
    \label{thm:regret_bound2}
    Consider the $1/2$-LC-Tsallis-INF. Assume that the loss is generated under an adversarial regime. Suppose that Assumptions~\ref{asm:contextual_dist} and \ref{asm:contextual_dist2}--\ref{asm:bounded_random} hold. Then, the regret satisfies 
    \[R_T = O\left(\sqrt{d \sqrt{K} T} + \frac{1}{\lambda}\left(\sqrt{K}
        +
        \frac{\sqrt{\sqrt{K}T}}{\sqrt{d}}
    \right) \right).\]
\end{theorem}
Note that this result does not require Assumptions~\ref{asm:contextual_dist_stochastic}, which restricts the context support to be finite.

Next, we show a regret upper bound in a stochastic regime with a margin condition, which depends on the parameter $\beta \in (0, \infty]$. Recall that $\beta$ decides the behavior of $\Delta(x)$. The proof is in Appendix~\ref{appdx:regret_upper_margin}.
\begin{theorem}[Regret upper bound in a stochastic regime with a margin condition]
\label{thm:regret_upper_margin}
Consider the $1/2$-LC-Tsallis-INF.
Assume that the loss is generated under a stochastic regime with a margin condition (Definition~\ref{def:margin}). Suppose that Assumption~\ref{asm:contextual_dist}--\ref{asm:bounded_random} hold. Then, the regret satisfies 
\[
    R_T = O\left(
        \left(\frac{1 + \beta}{\beta\Delta_*}\right)^{\frac{\beta}{2+\beta}}\left(L d \sqrt{K} \log(T)\right)^{\frac{1+\beta}{2+\beta}}T^{\frac{1}{2+\beta}}
        +
        \kappa
    \right),
\] 
where $\kappa =
        O\left(
            \sqrt{\frac{KL}{\lambda}}
            +
            \frac{\sqrt{K}L}{\lambda d} \log (T)
        \right)$.
\end{theorem}
For example, when $\beta = \infty$, then the regret is 
$R_T = O\left(\frac{1}{\Delta_*}L d \sqrt{K} \log(T)\right)$.
When $\beta = 1$, then the regret is $R_T = 
        O\left(\left(\frac{1}{\Delta_*}\right)^{\frac{1}{3}}\left(Ld \sqrt{K} \log(T)\right)^{\frac{2}{3}}T^{\frac{1}{3}}\right)$. 

Lastly, we derive a regret upper bound in a linear contextual adversarial regime with a self-bounding constraint, which is a generalization of the stochastic and adversarial regimes under $\beta = \infty$ in a margin condition. We provide the proof in Appendix~\ref{appdx:regret_bound3}.
\begin{theorem}[Regret upper bound in an adversarial regime with a self-bounding constraint]
    \label{thm:regret_bound3}
    Consider the $1/2$-LC-Tsallis-INF.
    Assume that the loss is generated under a linear contextual adversarial regime with a self-bounding constraint (Definition~\ref{def:bounding}). Suppose that Assumption~\ref{asm:contextual_dist}--\ref{asm:bounded_random} hold. Then,
        \[
        R_T = O\left(\frac{L d \sqrt{K}}{\Delta_*}\log(T) + \kappa + C  \right)
    \]
    holds. 
    Moreover,
    for $\frac{dL\sqrt{K}}{\Delta_*}\left(\log\left(\frac{T \Delta^2_*}{L^2 d \sqrt{K} }\right) + 1\right) \leq C \leq \frac{\Delta_* T}{L} $,
    the regret satisfies
        \[
        R_T = O\left(\sqrt{\frac{C  L\sqrt{K}d }{\Delta_*}\log_+\left(\frac{\Delta_*T}{CL}\right)}  + \kappa \right).
    \]
\end{theorem}

Theorem~\ref{thm:regret_bound3} implies that our algorithm achieves a tight regret bound of order $\log_+(T)$ with respect to $T$. Recall that $\log_+(x) = \max\{1, \log(x)\}$.

\subsection{Regret Transformation}
\label{sec:regret_transform}
This study considers a setting where arm-dependent features are observable, while there are also studies investigating arm-independent features \citep{Neu2020,kuroki2023bestofbothworlds,kato2023bestofbothworlds}. We show that (i) a regret derived under either of the problem settings can be transformed into a regret under the other setting, and (ii) transforming a regret under arm-dependent features to one under arm-independent features results in tighter bounds.

Here, we illustrate how to transform the setting with arm-independent features into one with arm-dependent features. We consider the former problems with $\widetilde{d}$-dimensional arm-independent features, $\widetilde{K}$ arms, and $\widetilde{T}$ rounds, denoted by P.Indep $\left(\widetilde{d}, \widetilde{K}, \widetilde{T}\right)$. We also consider the latter problem with $d$-dimensional features, $K$ arms, and $T$ rounds, denoted by P.Dep $(d, K, T)$.

In the P.Indep $\left(\widetilde{d}, \widetilde{K}, \widetilde{T}\right)$, we observe $\widetilde{d}$-dimensional features $\widetilde{\phi}(X_t)$ at each $t \in [T]$, and the loss follows a linear model $\widetilde{\ell}_t(a, X_t) = \left\langle \widetilde{\theta}_{a, t}, \widetilde{\phi}(X_t) \right\rangle + \widetilde{\varepsilon}_t$, where $\widetilde{\varepsilon}_t$ is an error term, $\widetilde{\phi}: \mathcal{X} \to \mathbb{R}^d$, and $\widetilde{\theta}_{a, t} \in \mathbb{R}^{\widetilde{d}}$. This problem can be transformed into P.Dep $\left(\widetilde{d}\widetilde{K}, \widetilde{K}, \widetilde{T}\right)$ by considering arm-dependent features \[\phi(a, X_t) = {\left(\bm{0}^\top_d \cdots \widetilde{\phi}^\top(X_t) \cdots \bm{0}^\top_d\right)^\top}\] and the parameters $\theta_t = {\left( \widetilde{\theta}^\top_{1,t} \cdots \widetilde{\theta}^\top_{a,t} \cdots \widetilde{\theta}^\top_{K,t}\right)^\top}$, where $\bm{0}_d$ is the $d$-dimensional zero vector. Here, the loss follows as $\ell_t(a, X_t) = \left\langle \theta_t, \phi(a, X_t)\right\rangle + \varepsilon_t$. Similarly, we can transform the P.Dep $(d, K, T)$ into the P.Indep $(dK, K, T)$.

Based on this fact, we have the following theorem. 
    \begin{theorem}
    \label{thm:reg_trans}
         If there exists an algorithm whose regret is $R_T = f(d, K, T)$ in the P.Dep $(d, K, T)$, where $f:\mathbb{N}\times \mathbb{N}\times \mathbb{N}\to \mathbb{R}$, then there exists an algorithm whose regret is given as $\widetilde{R}_T = f(\widetilde{d}\widetilde{K}, \widetilde{K}, \widetilde{T})$ in the P.Indep $(\widetilde{d}, \widetilde{K}, \widetilde{T})$. Conversely, if there exists an algorithm whose regret is given as $\widetilde{R}_T = \widetilde{f}\left(\widetilde{d}, \widetilde{K}, \widetilde{T}\right)$ in the P.Indep $\left(\widetilde{d}, \widetilde{K}, \widetilde{T}\right)$, where $\widetilde{f}:\mathbb{N}\times \mathbb{N}\times \mathbb{N}\to \mathbb{R}$ is some function of $\widetilde{d}$, $\widetilde{K}$, and $\widetilde{T}$, then there exists an algorithm whose regret is given as $R_T = \widetilde{f}\left(dK, K, T\right)$ in the P.Dep $(d, K, T)$.
    \end{theorem}

By using Theorem~\ref{thm:reg_trans}, we can transform regrets derived under each problem setting. This implies that if we derive a regret under one setting, we can obtain a regret under the other setting.

However, we note that we can derive tighter regrets by transforming a regret under the arm-dependent feature setting to a regret under the arm-independent feature setting. We confirm this by examining the regrets of the FTRL with Shannon entropy in the stochastic regime.

In the P.Indep $\left(\widetilde{d}, \widetilde{K}, \widetilde{T}\right)$, \citet{kuroki2023bestofbothworlds} and \citet{kato2023bestofbothworlds} show that the FTRL with Shannon entropy incurs $O\left(\frac{\widetilde{K}}{\Delta_{*}} \left(\frac{1}{ \lambda} + \widetilde{d}\right)\log(\widetilde{T})\log\left(\widetilde{K}\widetilde{T}\right)\right)$ regret, where we changed $\frac{\log(\widetilde{T})}{\lambda}$ to $\frac{1}{\lambda}$ since $\Sigma^{-1}_t$ is assumed to be exactly computable. As discussed above, this algorithm incurs $O\left(\frac{K}{\Delta_{*}} \left(\frac{1}{ \lambda} + dK\right)\log(T)\log\left(KT\right)\right)$ regret in the P.Dep $(d, K, T)$. 

In contrast, in the P.Dep $\left(d, K, T\right)$, our Theorem~\ref{cor:adversarial_regime} shows that the FTRL with Shannon entropy incurs $O\left(\frac{1}{\Delta_{*}} \left(\frac{1}{ \lambda} + d\right)\log(T)\log(KT)\right)$ regret, which implies that it incurs $O\left(\frac{1}{\Delta_{*}} \left(\frac{1}{ \lambda} + \widetilde{d}\widetilde{K}\right)\log(\widetilde{T})\log(\widetilde{K}\widetilde{T})\right)$ regret in the P.Indep $\left(\widetilde{d}, \widetilde{K}, \widetilde{T}\right)$.

By comparing the results, we find that we can obtain a tighter upper bound when transforming a regret under the arm-dependent setting to a regret under the arm-independent setting. In fact, if we transform a regret under the arm-independent setting to a regret under the arm-dependent setting, the order of the regret becomes $O(K^2)$ regarding $K$, which is $O(1)$ in our Theorem~\ref{cor:adversarial_regime}.

We present the proof of Theorem~\ref{thm:reg_trans} and more detailed arguments in Appendix~\ref{appdx:transform}.

\colorred

\subsection{Discussion}
Here, we discuss related topics.
\paragraph{Time complexity of solving the optimization problem $\eqref{eq:q_opt}$} The solution of \eqref{eq:q_opt} can be efficiently computed using the method described in Section 3.3 of \citet{Zimmert2019}. As mentioned there, it reduces the problem of solving a univariate equation, to which Newton’s method can be applied. This method has quadratic convergence, typically requiring fewer than ten iterations to converge with double precision.

\paragraph{Comparison with a simple baseline method} For comparison, we consider a baseline method where we run $|\mathcal{X}|$ independent $1/2$-Tsallis-INF algorithms, originally proposed for multi-armed bandits without contexts by \citet{Zimmert2019}, in each context. A similar analysis can be found in Section~18.1 of \citet{lattimore2020bandit}. In round $T$, for each context $x \in \mathcal{X}$, we define $T_x$ as the number of times that context $x$ is observed. From the results in \citet{Zimmert2019}, the regret of the $1/2$-Tsallis-INF for each context is given as $O\left(\frac{1}{\Delta_*}K\log(T_x)\right)$ in the stochastic regime and $O\left(\sqrt{KT_x}\right)$ in the adversarial regime. Therefore, if we marginalize the context-wise regret, the total regret is $\sum_{x \in \mathcal{X}} O\left(\frac{1}{\Delta_*}K\log(T_x)\right) = O\left(\frac{1}{\Delta_*}SK\log(T/S)\right)$ in the stochastic regime and $\sum_{x \in \mathcal{X}} O\left(\sqrt{KT_x}\right) = O\left(\sqrt{SKT}\right)$ in the adversarial regime, where $S = |\mathcal{X}|$ and $L \geq S$. In contrast, our $1/2$-LC-Tsallis-INF incurs $O\left(\frac{1}{\Delta_*}L \sqrt{K}\max\left\{d, \frac{1}{\lambda d}\right\} \log(T)\right)$ regret in the stochastic regime and $O\left(\sqrt{d \sqrt{K} T}\right)$ regret in the adversarial regime. Comparing these results, we find that the dependency on $K$ is improved in our $1/2$-LC-Tsallis-INF in both regimes, while the dependency on $L$ and $S$ ($L \geq S$) is worse than that of the baseline method. Additionally, our $1/2$-LC-Tsallis-INF depends on the feature dimension $d$, which does not appear in the baseline method.

Our setting can be reduced to a tabular Markov decision process (MDP). However, employing this reduction overlooks the linear structure in the loss function, ultimately compromising efficiency, particularly when the number of actions \(K\) is large. More specifically, in the stochastic setting, the upper bound derived under the tabular MDP interpretation is roughly \(O\left(\frac{KI \log T}{\Delta_{\min}}\right)\), corresponding to replacing \(d\sqrt{K}\) in our upper bound with \(K\), where $I$  is the number of interactions within each episode \citep{Jin2020simultaneouslylearning,dann2023best}. In the adversarial setting, the upper bound becomes \(O(\sqrt{KIT})\), corresponding to replacing \(d\sqrt{K}\) in our bound with \(KI\).

\section{Conclusion}
\label{sec:conclusion}
We presented a BoBW algorithm for linear contextual bandits with regret upper bounds that tightly depend on $T$. Our proposed algorithm, the $\alpha$-LC-Tsallis-INF, employs FTRL with Tsallis entropy regularization and achieves $O(\log(T))$ regret in the stochastic regime. Additionally, we derived regret upper bounds under the margin condition and the arm-dependent setting. An important remaining challenge is to improve the dependence on $T$ when $\beta \in (0, \infty)$ in the stochastic regime.

\bibliography{LCTsallisINF.bbl} 
\bibliographystyle{icml2024} 

\onecolumn 

\appendix

\tableofcontents

\section{Preliminaries for the proof of Theorems~~\ref{thm:regret_bound}--\ref{thm:regret_bound3}}
\label{sec:proof1}
This section provides preliminary results for the proof of Theorems~~\ref{thm:regret_bound}--\ref{thm:regret_bound3}. We prove Theorems~~\ref{thm:regret_bound}--\ref{thm:regret_bound3} in Appendix~\ref{appdx:proof_main}.

\subsection{Upper bound by Bregman divergence}
Let $X_0$ be a sample from the context distribution $\mathcal{G}$ independent of $\mathcal{F}_T$. Let $D(p, q)$ denote a Bregman divergence between $p. q\in\Pi$ with respect to $\psi_t$, defined as
\begin{align*}
    D(p, q) \coloneqq \psi(p) - \psi(q) - \Big\langle \nabla \psi(q), p - q \Big\rangle.
\end{align*}

In our proof, the following proposition plays an essential role, which is inspired by Lemma~4.4 in \citet{kato2023bestofbothworlds}.
\begin{proposition}
\label{lem:basic}
    Consider the $1/2$-LC-Tsallis-INF with our defined parameters. Then, the regret satisfies
    \begin{align*}
        R_T &\leq \mathbb{E}\Bigg[\sum^T_{t=1}\left( \gamma_t + \left\langle \widehat{\ell}_t(X_0), q_t(X_0) - q_{t+1}(X_0) \right\rangle - \frac{1}{\eta_t}D\big(q_{t+1}(X_0), q_t(X_0)\big) - \left(\frac{1}{\eta_t} - \frac{1}{\eta_{t-1}}\right)H\big(q_t(X_0)\big) \right) \Bigg],
    \end{align*}
    where we define
    \[H\big(q(x)\big) \coloneqq 2\left(-\sum_{a\in[K]}\sqrt{q(a\mid x)} + 1\right)\]
    and define $1/\eta_{0} = 0$ as an exception.
\end{proposition}
We show the proof below.

\begin{proof}[Proof of Proposition~\ref{lem:basic}]
Let us define an optimal policy $\pi^*\in \Pi$ as $\pi^*(\rho^*_T(x)\mid x) = 1$ and $\pi^*(a\mid x) = 0$ for all $a\in[K]\backslash \{\rho^*_T(x)\}$ and for all $x\in\mathcal{X}$.

Recall that in \eqref{eq:q_opt}, we defined $q_t$ as 
\[q_t(x) \coloneqq \argmin_{q\in \mathcal{P}_K}\left\{ \sum^{t-1}_{s=1}\left\langle \widehat{\ell}_s(x), q\right\rangle + \frac{1}{\eta_t}\psi\big(q\big)\right\}\]
for all $t \geq 2$.

From the definition of our algorithm, we have
\begin{align}
    R_T &= \mathbb{E}\left[\sum^T_{t=1}\left( \ell_t(A_t, X_t) - \ell_t\left(\rho^*_T(X_t), X_t\right)\right)\right]\nonumber\\
    &= \mathbb{E}\left[\sum^T_{t=1}\left\langle \ell_t(X_t), \pi_t(X_t) - \pi^*(X_t)\right\rangle \right]\nonumber\\
    &= \mathbb{E}\left[\sum^T_{t=1}\left\langle \ell_t(X_t), q_t(X_t) - \pi^*(X_t)\right\rangle + \sum^T_{t=1}\gamma_t \left\langle \ell_t(X_t), e^*(X_t) - q_t(X_t)\right\rangle \right]\nonumber\\
    &\leq \mathbb{E}\left[\sum^T_{t=1}\Big\langle \ell_t(X_t), q_t(X_t) - \pi^*(X_t)\Big\rangle + \sum^T_{t=1}\gamma_t \right]\nonumber\\
    &= \mathbb{E}\left[\sum^T_{t=1}\Big\langle \ell_t(X_0), q_t(X_0) - \pi^*(X_0)\Big\rangle + \sum^T_{t=1}\gamma_t \right]\nonumber\\
    &= \mathbb{E}\left[\sum^T_{t=1}\left\langle \widehat{\ell}_t(X_0), q_t (X_0)- \pi^*(X_0)\right\rangle + \sum^T_{t=1}\gamma_t \right] .
    \label{eq:target3}
\end{align}

Next, we show that for all $x\in\mathcal{X}$ and any $p^* \in \mathcal{P}_K$, it holds that
\begin{align}
\label{eq:target5}
    &\sum^T_{t=1}\left\langle \widehat{\ell}_t(x), q_t(x) - p^*(x) \right\rangle\nonumber\\
    &\leq \sum^T_{t=1}\left(\left\langle \widehat{\ell}_t(x), q_t(x) - q_{t+1}(x) \right\rangle - \frac{1}{\eta_t}D\big(q_{t+1}(x), q_t(x)\big) - \left(\frac{1}{\eta_t} - \frac{1}{\eta_{t-1}}\right)H\big(q_t(x)\big) \right).
\end{align}
This result follows from the definition of $q_t$; that is,
\begin{align*}
     &\left\langle \sum^T_{t=1}\widehat{\ell}_t(x), p^*(x) \right\rangle + \frac{1}{\eta_T}\psi\left(p^*(x)\right)\\
     &\geq \left\langle \sum^T_{t=1}\widehat{\ell}_t(x), q_{T+1}(x) \right\rangle + \frac{1}{\eta_T}\psi\left(q_{T+1}(x)\right)\\
     &\geq \left\langle \sum^{T-1}_{t=1}\widehat{\ell}_t(x), q_{T+1}(x) \right\rangle + \left\langle \widehat{\ell}_T(x), q_{T+1}(x) \right\rangle + \frac{1}{\eta_T}\psi\left(q_{T+1}(x)\right)\\
     &\geq \left\langle \sum^{T-1}_{t=1}\widehat{\ell}_t(x), q_{T}(x) \right\rangle + \left\langle \widehat{\ell}_T(x), q_{T+1}(x) \right\rangle + \frac{1}{\eta_T}\psi\left(q_{T}(x)\right) + \frac{1}{\eta_T}D\left(q_{T+1}, q_T\right)\\
     &\geq \sum^{T}_{t=1}\left( \left\langle \widehat{\ell}_t(x), q_{t+1}(x) \right\rangle + \left(\frac{1}{\eta_{t-1}} - \frac{1}{\eta_t}\right)\psi\left(q_{t}(x)\right) + \frac{1}{\eta_t}D\left(q_{t+1}, q_t\right)\right).
\end{align*}
Here, we defined $H\left(q_{t}(x)\right)$ as $\psi\left(q_{t}(x)\right) = H\left(q_{t}(x)\right)$ under the $1/2$-LC-Tsallis-INF.
Combining \eqref{eq:target5} with \eqref{eq:target3} yields the statement.
\end{proof}

\subsection{Bregman divergence associated with the Tsallis entropy} 
The Bregman divergence associated with the $\alpha$-Tsallis entropy 
\[\psi(q(x)) = - \frac{1}{\alpha}\left(\sum_{a\in[K]}q_t(a\mid x)^\alpha + 1\right)\] is given as
\begin{align*}
    D\big(p(x), q(x)\big) &= \frac{1}{\alpha}\sum_{a\in[K]}\left(q(a\mid x)^\alpha + \alpha\big(p(a\mid x) - q(a\mid x)\big)q(a\mid x)^{\alpha-1} - p(a\mid x)^\alpha\right)\\
    &= \sum_{a\in[K]}d\big(p(a\mid x), q(a\mid x)\big),
\end{align*}
where $d(p, q)$ is defined as
\begin{align*}
    d\left(p, q\right) \coloneqq \frac{1}{\alpha}q^\alpha + (p - q)q^{\alpha - 1} - \frac{1}{\alpha}p^\alpha \leq \frac{1-\alpha}{2}\big(\min\{p, q\}\big)^{\alpha - 2}\big(p - q\big)^2
\end{align*}
for all $p, q \in (0, 1)$.

\subsection{Upper bound of \texorpdfstring{$\left|\widehat{\ell}_t(a, x)\right|$}{f}}
\label{appdx:bound1}
For any $d \times d$ real symmetric matrices $A$ and $B$,
denote $A \succeq B$ if and only if $A - B$ is positive-semidefinite
and 
$A \succ B$ if and only if $A - B$ is positive-definite.

\begin{lemma}
    \label{lem:xAx}
    For any positive-definite matrix $A \in \mathbb{R}^{d\times d}$
    and a real vector $x \in \mathbb{R}^d$,
    if $A \succeq x x^{\top}$,
    it holds that $x A^{-1} x \le 1$.
\end{lemma}
\begin{proof}
    From the assumption of $A \succeq x x^{\top}$,
    we have
    \begin{align*}
        x^\top A^{-1} x
        =
        (A^{-1} x)^\top A (A^{-1} x)
        \ge
        (A^{-1} x)^\top x x^\top (A^{-1} x)
        =
        (x^\top A^{-1} x)^2,
    \end{align*}
    which implies 
    $x^\top A^{-1} x \le 1$.
\end{proof}

\begin{lemma}
    \label{lem:upper_loss}
    Suppose that Assumptions~\ref{asm:contextual_dist}--\ref{asm:bounded_random} hold.
    Then, for all $x \in\mathcal{X}$ and $a \in [K]$,
    it holds that
    \begin{align}
        \left|\widehat{\ell}_t(a, x)\right| 
        \leq 
        \min \left\{
            \frac{1}{\lambda \gamma_t},
            \sqrt{
                \frac{1}{\lambda \gamma_t (1-\gamma_t) q_t(a \mid x)g(x)}
            }
        \right\} .
        \label{eq:upper_loss}
    \end{align}
\end{lemma}
\begin{proof}
    From Assumption~\ref{asm:bounded_random} and the definition of $\widehat{\ell}_t$ in \eqref{eq:defthetahat},
    we have
    \begin{align}
        \nonumber
        \left|\widehat{\ell}_t(a, x)\right| 
        &= \left| \left\langle \widehat{\theta}_t, \phi(a, x)\right\rangle  \right|
        \leq \left| \left(\phi(a, x)\right)^\top \Sigma^{-1}_t \phi(A_t, X_t) \right|
        \\
        &\leq \sqrt{ \left(\phi(a, x)\right)^\top \Sigma^{-1}_t \phi(a, x) \cdot  \left(\phi(A_t, X_t)\right)^\top \Sigma^{-1}_t \phi(A_t, X_t)}.
        \label{eq:upper_loss01}
    \end{align}
    From Assumption~\ref{asm:contextual_dist2},
    for any $z \in \mathcal{Z}$,
    we have
    \begin{align}
        z^\top \Sigma_{t}^{-1} z
        \le
        z^\top \left( \gamma_t{\Sigma(e^*)} \right)^{-1} z
        \le
        \frac{1}{\gamma_t \lambda}.
        \label{eq:upper_loss02}
    \end{align}
    In addition,
    as we have
    \begin{align*}
        \Sigma_t
        &
        \succeq
        (1 - \gamma_t) \Sigma(q_{t})
        =
        (1 - \gamma_t) \sum_{x' \in \mathcal{X}} g(x') \sum_{ a' \in [K] } q_t(a' \mid x') \phi(a', x') (\phi(a', x'))^{\top}
        \\
        &
        \succeq
        (1 - \gamma_t) g(x) q_t(a \mid x) \phi(a, x) (\phi(a, x))^{\top},
    \end{align*}
    Lemma~\ref{lem:xAx} yields
    \begin{align}
        (\phi(a,x))^{\top} \Sigma_t^{-1} \phi(a, x) \le \frac{1}{(1-\gamma_t) q_t(a \mid x)g(x)}.
        \label{eq:upper_loss03}
    \end{align}
    Combining \eqref{eq:upper_loss01}, \eqref{eq:upper_loss02}, and \eqref{eq:upper_loss03},
    we obtain \eqref{eq:upper_loss}.
\end{proof}


\subsection{Upper bound of \texorpdfstring{$\mathbb{E}\left[\pi_t(a\mid X_0)\widehat{\ell}_t(a\mid X_0)^2\mid \mathcal{F}_{t-1}\right]$}{f}}
\label{appdx:bound2}
We upper bound $\mathbb{E}\left[\pi_t(a\mid X_0)\widehat{\ell}_t(a\mid X_0)^2\mid \mathcal{F}_{t-1}\right]$. This result is inspired by Lemma~6 in \citet{Neu2020}. 
\begin{lemma}
\label{prp:lemma6_neu}
   For all $t\in[T]$ and all $a\in[K]$, our policy satisfies
    \begin{align*}
        \mathbb{E}\left[\sum_{a\in[K]}\pi_t(a\mid X_0)\widehat{\ell}_t(a, X_0)^2\mid \mathcal{F}_{t-1}\right] \leq d. 
    \end{align*}
\end{lemma}
\begin{proof}
    We can show the statement as follows:
    \begin{align*}
        &\mathbb{E}\left[\sum_{a\in[K]}\pi_t(a\mid X_0) \hat{\ell}_t(a, X_0)^2\mid \mathcal{F}_{t-1}\right]\\
        &= \mathbb{E}\left[ \sum_{a\in[K]} \pi_t(a\mid X_0) \hat{\theta}^\top_t \phi(a, X_0)\phi^\top(a, X_0)\hat{\theta}_t \mid \mathcal{F}_{t-1}\right]\\
        &= \mathbb{E}\left[\hat{\theta}^\top_t \Sigma_t \hat{\theta}_t\mid \mathcal{F}_{t-1}\right]\\
        &\le \mathbb{E}\left[\ell_t(A_t, X_t)\phi^\top(A_t, X_t) \Sigma^{-1}_t \Sigma_t \Sigma^{-1}_t \phi(A_t, X_t)\ell_t(A_t, X_t)\mid \mathcal{F}_{t-1}\right]\\
        &= \mathbb{E}\left[\phi^\top(A_t, X_t) \Sigma^{-1}_t \Sigma_t \Sigma^{-1}_t \phi(A_t, X_t)\mid \mathcal{F}_{t-1}\right]\\
        &= \mathbb{E}\left[\mathrm{tr}(\phi^\top(A_t, X_t) \Sigma^{-1}_t \Sigma_t \Sigma^{-1}_t \phi(A_t, X_t))\mid \mathcal{F}_{t-1}\right]\\
        &= \mathbb{E}\left[\mathrm{tr}(\Sigma^{-1}_t \phi(A_t, X_t) \phi^\top(A_t, X_t))\mid \mathcal{F}_{t-1}\right]\\
        &= \mathbb{E}\left[\mathrm{tr}(\Sigma^{-1}_t \Sigma_t)\mid \mathcal{F}_{t-1}\right]\\
        &= d.
    \end{align*}
\end{proof}

\subsection{Margin condition}
To utilize a margin condition for deriving a regret upper bound, we show the following lemma. 
\begin{lemma}
\label{lem:margin}
Let $U \in [0, 1]$ and $V \in \mathbb{R}$ be some random variables. Let $\Delta_* > 0$ be some universal constant.
For $\beta \in (0, \infty]$, if the random variable $U$ has a mean $\mathbb{E}[U] =\mu$ and the random variable $V$ satisfies
\begin{align*}
    F(h)\coloneqq \mathbb{P}\big(V \leq h\big) \leq \frac{1}{2}\left(\frac{h}{\Delta_*}\right)^\beta
\end{align*}
for all $h \in\left[0, \Delta_*\right]$, and $\mu \in \left[0, \frac{1}{2}\right]$ holds, then
    \begin{align*}
        \mathbb{E}\big[UV\big] \geq \frac{\Delta_* \beta }{ 2(1+\beta)}(2\mu)^{\frac{1+\beta}{\beta}}
    \end{align*}
    holds. 
\end{lemma}
\begin{proof}
    For simplicity, let us assume that $V$ has a density function. This assumption implies that the cumulative density function of $V$ is a continuous and monotonically increasing function. Then, $Z \coloneqq F(V)$ follows the uniform distribution $\mathrm{Unif}[0, 1]$ over the support $[0, 1]$. 
    
    Let us define
    \[
        G(z) \coloneqq \mathbb{E}[U\mid Z = z].
    \]
    Then, we have $G(z) \in [0, 1]$, and the expected values of $U$ and $UV$ are given as follows:
    \begin{align}
    \label{eq:lem_eq1}
        &\mathbb{E}\big[U\big] = \mathop{\mathbb{E}}_{W \sim  \mathrm{Unif}[0, 1]}[G(W)] = \int^1_0g(W)\mathrm{d}w = \mu,\\
       \label{eq:lem_eq2} 
        &\mathbb{E}\big[UV\big] = \mathop{\mathbb{E}}_{W \sim  \mathrm{Unif}[0, 1]}[G(W)F^{-1}(W)] = \int^1_0g(w)F^{-1}(w)\mathrm{d}w,
    \end{align}
    where $\mathop{\mathbb{E}}_{W \sim  \mathcal{J}}$ denotes an expectation operator for a random variable $W$ under a probability distribution $\mathcal{J}$. 

    Here, note that $F^{-1}$ is a monotonically decreasing function.
    Therefore, for $G(w)\in [0, 1]$,
    under \eqref{eq:lem_eq1},
    a function $G^*$ minimizes \eqref{eq:lem_eq2} if $G^*(w) \coloneqq \mathbbm{1}[w\leq \mu]$. 

    In conclusion, we obtain 
    \begin{align}
    \label{eq:conclu1}
        \mathbb{E}[UV]\geq \int^1_0 G^*(w)F^{-1}(w)\mathrm{d}z = \int^{\mu}_0F^{-1}(w)\mathrm{d}w.
    \end{align}
    Furthermore, because we assumed $F(h)\leq \frac{1}{2}\left(\frac{h}{\Delta_*}\right)^\beta$ for all $h \in\left[D, \Delta_*\right]$, we have $F^{-1}(w) \geq \Delta_*(2w)^{\frac{1}{\beta}}$ for $w \in [F(D), F(\Delta_*)] = \left[\frac{1}{2}\frac{D}{\Delta_*}, \frac{1}{2}\right]$, which implies 
    \begin{align}
    \label{eq:conclu2}
        \int^{\mu}_0F^{-1}(w)\mathrm{d}w \geq \int^{\mu}_0 \Delta_*(2w)^{\frac{1}{\beta}}\mathrm{d}w \geq \int^{\mu}_{\frac{1}{2}\frac{D}{\Delta_*}} \Delta_*(2w)^{\frac{1}{\beta}}\mathrm{d}w = \frac{\Delta_*}{2(1+\beta)}(2\mu)^{\frac{1+\beta}{\beta}},
    \end{align}
    where we used $0 \leq \mu \leq \frac{1}{2}$. 
    From \eqref{eq:conclu1} and \eqref{eq:conclu2}, we obtain the statement.
\end{proof}

\section{Proof of Theorems~~\ref{thm:regret_bound}--\ref{thm:regret_bound3}}
\label{appdx:proof_main}
This section provides the proofs for Theorem~\ref{thm:regret_bound}. In Appendix~\ref{sec:proof1}, we provide preliminary results for the proof. Then, in Appendix~\ref{appdx:stab_penal}, we decompose the regret into the stability and penalty terms. Lastly, in Appendix~\ref{sec:proof2}, we prove Theorem~\ref{thm:regret_bound}.

\subsection{Stability and penalty decomposition}
\label{appdx:stab_penal}
Following the standard analysis of FTRL methods, we decompose the regret into stability and penalty terms. 

Based on the result in Proposition~\ref{lem:basic}, let us define $r_T(x) $ as
\begin{align*}
    r_T(x) 
    & \coloneqq
        \sum^T_{t=1}\left( \gamma_t + \left\langle \widehat{\ell}_t(x), q_t(x) - q_{t+1}(x) \right\rangle- \frac{1}{\eta_t}D(q_{t+1}(x), q_t(x)) \right. 
    \left. - \left(\frac{1}{\eta_t} - \frac{1}{\eta_{t-1}}\right)H\big(q_t(x)\big) \right) .
\end{align*}
Using this function $r_T(x) $, we can bound $R_T$ as 
$
    R_T \leq \mathbb{E}\big[r_T(X_0)\big] .
$

The standard FTRL analysis breaks the pseudo-regret into penalty and stability terms. Following this approach, we decompose the pointwise regret upper bound $r_T(x)$ as follows: 
\begin{align}
\label{eq:stab_penal_decom}
    r_T(x) 
    &= \annot{
        \sum^T_{t=1} \left( \gamma_t + \left\langle \widehat{\ell}_t(x), q_t(x) - q_{t+1}(x) \right\rangle - \frac{1}{\eta_t}D(q_{t+1}(x), q_t(x)) \right) 
        }{=\ stability term}
    \nonumber\\
    &\ \ \ \ \ \ \ \ \ \ \ \ \ \ \ \ \ \ \ \ \ 
    +
    \annot{
            \left(
            - 
            \sum^T_{t=1}\left(\frac{1}{\eta_t} - \frac{1}{\eta_{t-1}}\right)H\big(q_t(x)\big)
            \right)
    }{=\ penalty term}.
\end{align}

Our remaining task is to derive upper bounds for the following terms and for all $x\in\mathcal{X}$:
\begin{align}
\label{eq:stability_term}
    \mathrm{stability}_t(x) &\coloneqq \gamma_t + \left\langle \widehat{\ell}_t(x), q_t(x) - q_{t+1}(x) \right\rangle - \frac{1}{\eta_t}D(q_{t+1}(x), q_t(x)),\\
    \label{eq:penalty_term}
    \mathrm{penalty}_t(x) & \coloneqq - \left(\frac{1}{\eta_t} - \frac{1}{\eta_{t-1}}\right)H\big(q_t(x)\big),
\end{align}

In Appendices~\ref{appdx:bound_stab}--~\ref{appdx:bound_penal}, we bound the stability and penalty terms, respectively.

\subsection{Bounding the stability term}
\label{appdx:bound_stab}
To bound the stability term, we obtain the following lemma. Recall that for all $x\in\mathcal{X}$ and $q_t(x)\in\mathcal{P}_K$, we defined $a_t^\dagger(x) \in\argmax_{a\in[K]}q_t(a\mid x)$.
\begin{lemma}
    \label{lem:upper_stability_sup2}
    Consider the $\alpha$-LC-Tsallis-INF.
    Suppose that Assumptions~\ref{asm:bounded_random}--\ref{asm:contextual_dist}, and \ref{asm:contextual_dist2} hold.
    If 
    \[
        \eta_t |\widehat{\ell}_t(a, x)|
        \leq
        \frac{1-\alpha}{4}\left(q_t(a \mid x)\right)^{\alpha - 1}
    \]
    holds for all 
    $a \in [K]
    $ 
    and for all $x \in \mathcal{X}$,
    we then have
    \begin{align*}
        &\left\langle \widehat{\ell}_t(x), q_t(x) - q_{t+1}(x) \right\rangle - \frac{1}{\eta_t}D(q_{t+1}(x), q_t(x))\\
        &\leq \frac{4 \eta_t }{1 - \alpha}\left(\sum_{a \neq a_t^\dagger(x)}\left(q_t(a\mid x)\right)^{2-\alpha}\ell^2_t(a\mid x) + \left(\min\big\{q_t(a_t^\dagger(x)\mid x), 1 - q_t(a_t^\dagger(x)\mid x)\big\}\right)^{2-\alpha}\ell^2_t(a_t^\dagger(x), x)\right).
    \end{align*}
\end{lemma}
The proof is shown in Appendix~\ref{appdx:lem:upper_stability_sup1}.
By using this lemma,
we show an upper bound for the stability term.
To provide it, we introduce the following notations: for each $t\in[T]$,
\begin{align*}
    &a^\dagger_t(x) \coloneqq \argmax_{a\in[K]} q_t(a\mid x),\\
    &\omega_t \coloneqq \sup_{x\in\mathcal{X}}\min_{a\in[K]}\left\{ 1 - q_t(a\mid x) \right\} = \sup_{x\in\mathcal{X}}\left\{1 - q_t\left(a^\dagger_t(x)\mid x\right)\right\}.
\end{align*}

\begin{lemma}[Upper bound for the stability term]
\label{lem:upper_stability} 
    Consider the $1/2$-LC-Tsallis-INF. Assumptions~\ref{asm:bounded_random}--\ref{asm:contextual_dist}, and \ref{asm:contextual_dist2} hold.
    It holds for all $x \in \mathcal{X}$ and $t \in [T]$ that
    \begin{align*}
        &\mathrm{stability}_t(x) \leq 2\eta_t\sqrt{\omega_t}\sum_{a\in[K]}\widehat{\ell}_t(a\mid x)^2\pi_t(a\mid x).
    \end{align*}
\end{lemma}
\begin{proof} 
We first check the condition of Lemma~\ref{lem:upper_stability_sup2}.
From Lemma~\ref{lem:upper_loss}
and the definition of $\gamma_t$ in \eqref{eq:q_opt},
we have
\begin{align*}
    \eta_t \left|\widehat{\ell}_t(a, x)\right| \le 
    \eta_t
    \sqrt{ \frac{1}{  \lambda \gamma_t (1 - \gamma_t) q_t(a \mid x) g(x)  }}
    \le
    \eta_t
    \sqrt{ \frac{2L}{ \lambda \gamma_t  q_t(a \mid x) }}
    \le
    \frac{1}{8}
    \left(q_t(a \mid x)\right)^{-1/2}.
\end{align*}
Hence,
for any $x \in \mathcal{X}$
we can apply Lemma~\ref{lem:upper_stability_sup2} with $\alpha = 1/2$ to obtain
the following:
\begin{align*}
    &\left\langle \widehat{\ell}_t(x), q_t(x) - q_{t+1}(x) \right\rangle - \frac{1}{\eta_t}D(q_{t+1}(x), q_t(x))\\
    &= \frac{1}{\eta_t}\left( \left\langle \eta_t\widehat{\ell}_t(x), q_t(x) - q_{t+1}(x) \right\rangle - D(q_{t+1}(x), q_t(x))\right)\\
    &\leq \eta_t\left( \sum_{a\neq a^\dagger_t(x)}\left(\widehat{\ell}_t\big(a\mid x\big)\right)^2\left(q_t\big(a\mid x\big)\right)^{3/2} + \left(\widehat{\ell}_t\big(a^\dagger_t(x) \mid x\big)\right)^2\min\Big\{q_t\big(a^\dagger_t(x) \mid x\big), 1 - q_t\big(a^\dagger_t(x) \mid x\big)\Big\}^{3/2}\right)\\
    &\leq \eta_t\left(\sum_{a\neq a^\dagger_t(x)}\left(\widehat{\ell}_t\big(a\mid x\big)\right)^2q_t\big(a\mid x\big)\sqrt{\omega_t} + \left(\widehat{\ell}_t\big(a^\dagger_t(x) \mid x\big)\right)^2q_t\left( a^\dagger_t(x) \right) \Big(1 - q_t\big(a^\dagger_t(x) \mid x\big)\Big)^{1/2}\right)\\
    &\leq \sqrt{\omega_t}\eta_t \sum_{a\in[K]}\widehat{\ell}^2_t(a\mid x)q_t(a\mid x) \leq 2\sqrt{\omega_t}\eta_t\sum_{a\in[K]}\widehat{\ell}^2_t(a\mid x) \pi_t(a\mid x),
\end{align*}
where the second inequality follows from
    $
    q_t(a \mid x) \le 1 - q_t(a^{\dagger}_t(x) \mid x) \le \omega_t
    $
that holds for all $x \in \mathcal{X}$ and $a \in [K] \setminus \{ a^{\dagger}_t(x) \}$,
and the last inequality follows from $\pi_t \geq (1 - \gamma_t) q_t \geq q_t/2$.

This completes the proof.
\end{proof}

\subsection{Bounding the penalty term}
\label{appdx:bound_penal}
Next, we bound $\mathrm{penalty}_t(x)$. 
\begin{lemma}[Upper bound of the penalty term]
\label{lem:upper_penalty}
For all $\alpha \in [0, 1]$, the penalty term of $\alpha$-LC-Tsallis-INF satisfies
    \begin{align*}
        \mathrm{penalty}_t(x) &= - \left(\frac{1}{\eta_t} - \frac{1}{\eta_{t-1}}\right)H\big(q_t(x)\big)\leq 2\left(\frac{1}{\eta_t} - \frac{1}{\eta_{t-1}}\right)\sqrt{K\omega_t},
    \end{align*}
    where recall that the penalty term $\mathrm{penalty}_t(x)$ is defined in \eqref{eq:penalty_term}. 
\end{lemma}
\begin{proof}
    This statement directly follows from the following inequality: for all $x\in\mathcal{X}$,
    \begin{align*}
        - H(q_t(x)) &= 2\left(\sum_{a\in[K]}\sqrt{q_t(a\mid x)} - 1\right)\\
        &\leq 2\sum_{a\neq a^\dagger(x)}\sqrt{q_t(a\mid x)}\\
        &\leq 2\sqrt{(K-1)\sum_{a\neq a^\dagger(x)}q_t(a\mid x)}\\
        &= 2\sqrt{(K-1)(1-q_t(a^\dagger_t(x)\mid x))} \leq 2\sqrt{K\omega_t}.
    \end{align*}
\end{proof}

\subsection{Proof of Theorem~\ref{thm:regret_bound}}
\label{sec:proof2}

Then, we prove Theorem~\ref{thm:regret_bound} as follows.  
\begin{proof}
From Lemma~\ref{lem:upper_stability}, 
we have
\begin{align*}
    \mathrm{stability}_t(x) \leq 2\eta_t\sqrt{\omega_t}\sum_{a\in[K]}\widehat{\ell}^2_t(a\mid x)\pi_t(a\mid x).
\end{align*}
    
From Lemma~\ref{lem:upper_penalty}, we have
\begin{align*}
    \mathrm{penalty}_t(x) \leq 2\left(\frac{1}{\eta_t} - \frac{1}{\eta_{t-1}}\right)\sqrt{K \omega_t}.
\end{align*}
Therefore,
we can bound the pointwise regret as
\begin{align*}
    r_T(x) &\leq
    \sum_{t=1}^T
    \left(
    2\sqrt{\omega_t}\eta_t\sum_{a\in[K]}\widehat{\ell}^2_t(a\mid x) \pi_t(a\mid x) + 2\left(\frac{1}{\eta_t} - \frac{1}{\eta_{t-1}}\right)\sqrt{K\omega_t}  + \gamma_t
    \right).
\end{align*}

We hence have
\begin{align}
    \nonumber
    R_T 
    &\leq \sum^T_{t=1}\mathbb{E}\left[2 \sqrt{\omega_t}\eta_t\sum_{a\in[K]}\widehat{\ell}_t(a\mid X_0)^2 \pi_t(a\mid X_0) + 2\left(\frac{1}{\eta_t} - \frac{1}{\eta_{t-1}}\right)\sqrt{K\omega_t}  + \gamma_t\right]
    \\
    &
    \le \sum^T_{t=1}\mathbb{E}\left[2d\sqrt{\omega_t}\eta_t
    + 2\left(\frac{1}{\eta_t} - \frac{1}{\eta_{t-1}}\right)\sqrt{K\omega_t}  + \gamma_t\right] ,
    \nonumber
    \\
    &
    \le 
    O \left(
        \mathbb{E}\left[
            \sum^T_{t=1} d \eta_t \sqrt{\omega_t}
            +
            \sum^T_{t=1} \left(\frac{1}{\eta_t} - \frac{1}{\eta_{t-1}}\right)\sqrt{K} \sqrt{\omega_t} 
        \right]
        +
        \sum_{t=1}^T \gamma_t
    \right),
    \label{eq:regomega0}
\end{align}
where the second inequality follows from Lemma~\ref{prp:lemma6_neu}.

Recall here that when Assumption~\ref{asm:contextual_dist_stochastic} holds, parameters $\eta_t$ and $\gamma_t$ are defined as
\begin{align*}
    \widetilde{\eta}_t \coloneqq \frac{K^{1/4}}{\sqrt{ d t} },
    \quad
    \eta_t \coloneqq \min \left\{ \widetilde{\eta}_t,~\frac{1}{16}\sqrt{\frac{\lambda}{L}} \right\}, 
    \quad
    \gamma_t \coloneqq \frac{128 L \eta_t^2}{\lambda} 
    \le
    \frac{128 \sqrt{K} L}{ \lambda d t}
\end{align*}
for $t=1,2, \ldots$ and $1/\eta_0 = 0$.
In this case, we have
\begin{align}
    \label{eq:regomega01}
    \sum^T_{t=1} d \eta_t \sqrt{\omega_t}
    \le
    \sum^T_{t=1} d \widetilde{\eta}_t \sqrt{\omega_t}
    =
    \sqrt{d}K^{1/4}
    \sum^T_{t=1}  \frac{\sqrt{\omega_t}}{\sqrt{t}},
\end{align}
and
\begin{align}
    \nonumber
    \sum^T_{t=1} \left(\frac{1}{\eta_t} - \frac{1}{\eta_{t-1}}\right)\sqrt{K} \sqrt{\omega_t} 
    &
    \le
    \sum^T_{t=2} \left(\frac{1}{\eta_t} - \frac{1}{\eta_{t-1}}\right)\sqrt{K} \sqrt{\omega_t} 
    +
    \frac{1}{\eta_1} \sqrt{K} \omega_1
    \\
    \nonumber
    &
    \le
    \sum^T_{t=2} \left(\frac{1}{\widetilde{\eta}_t} - \frac{1}{\widetilde{\eta}_{t-1}}\right)\sqrt{K} \sqrt{\omega_t} 
    +
    \max
    \left\{
        \frac{1}{\widetilde{\eta}_1} ,
        16 \sqrt{\frac{L}{\lambda}}
    \right\}
    \sqrt{K} \omega_1
    \\
    &
    =
    O\left(
        \sqrt{d} K^{1/4}
        \sum^T_{t=1}  \frac{\sqrt{\omega_t}}{\sqrt{t}}
        +
        \sqrt{\frac{KL}{\lambda}}
    \right).
    \label{eq:regomega02}
\end{align}
It also holds that
\begin{align}
    \sum_{t=1}^T \gamma_t
    \le
    \frac{128\sqrt{K}L}{ \lambda d}
    \sum_{t=1}^T
    \frac{1}{t}
    =
    O \left(
    \frac{\sqrt{K}L}{\lambda d}
    \log T
    \right).
    \label{eq:regomega03}
\end{align}
When Assumption~\ref{asm:contextual_dist_stochastic} holds, by combining \eqref{eq:regomega0}, \eqref{eq:regomega01}, \eqref{eq:regomega02}, and \eqref{eq:regomega03},
we obtain
\begin{align*}
    R_T= O\left(\mathbb{E}\left[\sum^T_{t=1}\frac{\sqrt{\sqrt{K} d \omega_t}}{\sqrt{t}}\right]
        +
        \sqrt{\frac{KL}{\lambda}}
        +
        \frac{\sqrt{K}L}{\lambda d} \log T
    \right).
\end{align*}

When Assumption~\ref{asm:contextual_dist_stochastic} does not hold, recall that $\eta_t$ is defined as
\[\eta_t \coloneqq \min \left\{ \widetilde{\eta}_t,~\frac{\lambda}{16} \right\}.\] 
Using an argument similar to the case when Assumption~\ref{asm:contextual_dist_stochastic} holds, we obtain 
\begin{align*}
    R_T= O\left(\mathbb{E}\left[\sum^T_{t=1}\frac{\sqrt{\sqrt{K} d \omega_t}}{\sqrt{t}}\right]
        +
        \frac{1}{\lambda}\left(\sqrt{K}
        +
        \frac{\sqrt{\sqrt{K}T}}{\sqrt{d}}
    \right)\right).
\end{align*}

\end{proof}

\subsection{Proof of Theorem~\ref{thm:regret_bound2}}

\begin{proof}
    Recall that $\omega_t \leq 1$. By replacing $\omega_t$ with $1$ in Theorem~\ref{thm:regret_bound}, we can directly obtain a regret upper bound in Theorem~\ref{thm:regret_bound2}. 
\end{proof}

\subsection{Proof of Theorem~\ref{thm:regret_upper_margin}}
\label{appdx:regret_upper_margin}
\begin{lemma}
    \label{lem:omegaRmargin}
    Under the margin condition given in Definition~\ref{def:margin},
    we have
    \begin{align}
        \mathbb{E} \left[
            \sum_{t=1}^T \omega_t
        \right]
        \le
        L
        T^{\frac{1}{1+\beta}}
        \left(
            \frac{2 (1+\beta) R_T }{\beta \Delta^* }
        \right)^{\frac{\beta}{1+\beta}}.
    \end{align}
\end{lemma}
\begin{proof}
    Under the margin condition,
    we have
    \[
        R_T \geq \mathbb{E}\left[\sum^T_{t=1}\Delta(X_t) \sum_{a\neq \rho^*_T(X_t)}\pi_t(a\mid X_t)\right].
    \]
    From this,
    by using $a^\dagger_t$,
    the regret is lower bounded as 
    \begin{align*}
        R_T
        & \geq \mathbb{E}\left[\sum^T_{t=1}\Delta(X_t) \sum_{a\neq \rho^*_T(X_t)}\pi_t(a\mid X_t)\right] \geq \frac{1}{2} \mathbb{E}\left[\sum^T_{t=1}\Delta(X_t) \sum_{a\neq \rho^*_T(X_t)}q_t(a\mid X_t)\right]\\
        & \geq \frac{1}{2}\mathbb{E}\left[\sum^T_{t=1}\Delta(X_t) \big(1 - q_t(\rho^*_T(X_t)\mid X_t)\big)\right] \geq \frac{1}{2}
        \mathbb{E}\left[\sum^T_{t=1}\Delta(X_t) \big(1 - q_t(a^\dagger_t(X_t)\mid X_t)\big)\right],
    \end{align*}
    where we used $\pi_t(a\mid x) \geq (1-\gamma_t)q_t(a\mid x) \geq q_{t}(a\mid x)/2$ for all $a\in[K]$ and $x\in\mathcal{X}$.
    Define $u_t$ by
    \begin{align}
        \label{eq:defut}
        u_t
        =
        \mathbb{E} \left[ 1 - q_t \left( a_t^{\dagger}(X_t) \mid X_t \right) \right].
    \end{align}
    Then,
    from Lemma~\ref{lem:margin},
    under the margin condition given in Definition~\ref{def:margin},
    we have
    \begin{align*}
        \mathbb{E}\left[\Delta(X_t) \big(1 - q_t(a^\dagger_T(X_t)\mid X_t)\big)\right]
        \ge
        \Delta_* \frac{\beta}{1+\beta}
        \left(
            \mathbb{E} \left[ 1 - q_t \left( a_t^{\dagger}(X_t) \mid X_t \right) \right]
        \right)^{\frac{1+\beta}{\beta}}
        =
        \Delta_* \frac{\beta}{1+\beta} 
        u_t^{\frac{1+\beta}{\beta}} .
    \end{align*}
    Hence,
    under the margin condition,
    we have
    \begin{align}
        R_T
        &
        \ge
        \frac{\beta \Delta_*}{2(1+\beta)} 
        \sum_{t=1}^T
        u_t^{\frac{1+\beta}{\beta}} 
        =
        \frac{\beta \Delta_* T}{2(1+\beta)} 
        \frac{1}{T}
        \sum_{t=1}^T
        u_t^{\frac{1+\beta}{\beta}} 
        \\
        &
        \ge
        \frac{\beta \Delta_* T}{2(1+\beta)} 
        \left(
            \frac{1}{T}
            \sum_{t=1}^T u_t
        \right)^{\frac{1+\beta}{\beta}}
        =
        \frac{\beta \Delta_* T^{-\frac{1}{\beta}}}{2(1+\beta)} 
        \left(
            \sum_{t=1}^T u_t
        \right)^{\frac{1+\beta}{\beta}},
    \end{align}
    where the second inequality follows from Jensen's inequality.
    In addition,
    as we have
    \begin{align}
        \nonumber
        u_t 
        &
        = 
        \mathbb{E} \left[
            \sum_{x \in \mathcal{X}} g(x)
            \left( 1 - q_t \left( a_t^{\dagger}(x) \mid x \right) \right)
        \right]
        \ge
        \frac{1}{L}
        \mathbb{E} \left[
            \sum_{x \in \mathcal{X}} 
            \left( 1 - q_t \left( a_t^{\dagger}(x) \mid x \right) \right)
        \right]
        \\
        &
        \ge
        \frac{1}{L}
        \mathbb{E} \left[
            \sup_{x \in \mathcal{X}} 
            \left( 1 - q_t \left( a_t^{\dagger}(x) \mid x \right) \right)
        \right]
        =
        \frac{1}{L}
        \mathbb{E} \left[
            \omega_t
        \right],
        \label{eq:boundut}
    \end{align}
    we have
    \begin{align}
        R_T
        \ge
        \frac{\beta \Delta_* T^{-\frac{1}{\beta}}}{2(1+\beta)} 
        \left(
            \frac{1}{L}
            \mathbb{E} \left[
                \sum_{t=1}^T \omega_t
            \right]
        \right)^{\frac{1+\beta}{\beta}},
    \end{align}
    which implies
    \begin{align}
        \mathbb{E} \left[
            \sum_{t=1}^T \omega_t
        \right]
        \le
        L
        T^{\frac{1}{1+\beta}}
        \left(
            \frac{2 (1+\beta) R_T }{\beta \Delta^* }
        \right)^{\frac{\beta}{1+\beta}}.
    \end{align}
\end{proof}

\begin{proof}[Proof of Theorem~\ref{thm:regret_upper_margin}]
We start with the upper bound given in Theorem~\ref{thm:regret_bound}. 
From the Cauchy-Schwarz inequality and the Jensen inequality, we have
\begin{align*}
    R_T 
    &=  O\left(\mathbb{E}\left[\sum^T_{t= 1}\frac{\sqrt{\sqrt{K} d \omega_t}}{\sqrt{t}}\right] + \kappa \right)
    \\
    &=  O\left(\sqrt{\sqrt{K}  d }\sqrt{\sum^T_{t= 1}\frac{1}{t}}\sqrt{\sum^T_{t=1}\mathbb{E}\left[\omega_t\right]} + \kappa \right)
    \\
    &=  O\left(\sqrt{\sqrt{K} d \log(T)}\sqrt{\sum^T_{t= 1}\mathbb{E}\left[\omega_t\right]} + \kappa \right).
\end{align*}
From this and Lemma~\ref{lem:omegaRmargin},
we have
\begin{align}
    R_T 
    = 
    O\left(\sqrt{\sqrt{K} d \log(T)
    \cdot
        L
        T^{\frac{1}{1+\beta}}
        \left(
            \frac{ (1+\beta) }{\beta \Delta^* }
        \right)^{\frac{\beta}{1+\beta}}
    }
    R_T^\frac{\beta}{2(1+\beta)}
    + \kappa
    \right),
\end{align}
which implies
\begin{align}
    R_T
    =
    O\left(
        \left(
            \frac{ (1+\beta) }{\beta \Delta^* }
        \right)^{\frac{\beta}{2+\beta}}
        \left(
        L \sqrt{K} d 
        \log(T)
        \right)^{\frac{1+\beta}{2+\beta}}
        T^{\frac{1}{2+\beta}}
        + \kappa
    \right).
\end{align}
We here used the fact that
$x \le a x^v + b$
implies
$x = O\left( a^\frac{1}{1-v} + b \right)$,
which holds for any for $a>0$, $b \ge 0$, $x \ge 0$ and $v \in (0, 1)$,

\end{proof}

\subsection{Proof of Theorem~\ref{thm:regret_bound3}}
\label{appdx:regret_bound3}
\begin{proof}[Proof of Theorem~\ref{thm:regret_bound3}]
    From Definition~\ref{def:bounding}
    and the fact that $\pi_t(a \mid X_t) \ge \gamma_t q_t(a \mid X_t) \ge \frac{1}{2}q_t(a \mid X_t) $,
    we have
    \begin{align*}
        R_T 
        &\geq \mathbb{E}\left[\sum^T_{t= 1}\sum_{a\neq \rho^*_T(X_t)}\Delta_t(a\mid X_t) \pi_t(a\mid X_t)\right] - C\\
        &\geq \frac{1}{2}\mathbb{E}\left[\sum^T_{t= 1}\sum_{a\neq \rho^*_T(X_t)}\Delta_t(a\mid X_t) q_t(a\mid X_t)\right] - C\\
        &\geq \frac{1}{2}\Delta_*\mathbb{E}\left[\sum^T_{t= 1}\sum_{a\neq \rho^*_T(X_t)}q_t(a\mid X_t)\right] - C\\
        &\geq \frac{1}{2}\Delta_*\mathbb{E}\left[\sum^T_{t=1}\left( 1 - q_t(a^{\dagger}_t(X_t) \mid X_t) \right) \right] - C\\
        &= 
        \frac{1}{2}\Delta_* \sum_{t=1}^T u_t - C
        \ge
        \frac{\Delta_*}{2L} \mathbb{E} \left[ \sum_{t=1}^T \omega_t \right] - C,
    \end{align*}
    where the third inequality follows from the assumption that $\Delta(a \mid X_t) \ge \Delta_*$ holds for all $a \neq \rho^*_T(x)$ in posed in Definition~\ref{def:bounding}.
    The value of $u_t$ is defined in \eqref{eq:defut}, and the last inequality follows from \eqref{eq:boundut}.
    From this and Theorem~\ref{thm:regret_bound},
    by applying Theorem~4 in \citet{Masoudian2021} with $K = 2$, $\Delta_i = \frac{\Delta_*}{L}$ and $B = O( \sqrt{d \sqrt{K}} )$,
    we obtain 
    \begin{align*}
        R_T 
        &= O\left(\frac{d L\sqrt{K}}{\Delta_*}
            \log \left( \frac{\Delta_*^2 T}{L^2} \right) 
            + \kappa + C
        \right).
    \end{align*}
    Moreover,
    for $\frac{dL\sqrt{K}}{\Delta_*}\left(\log\left(\frac{T \Delta^2_*}{L^2 d \sqrt{K} }\right) + 1\right) \leq C \leq \frac{\Delta_* T}{L} $, we have
    \begin{align*}
        R_T &= O\left(\sqrt{\frac{C  L\sqrt{K}d }{\Delta_*}} \left(\sqrt{\log\left(\frac{\Delta_*T}{CL}\right)} + 2\right)
         + W + \kappa \right),
    \end{align*}
    where $W$ is a subdominant term given as
    \[
        W \coloneqq \frac{L\sqrt{K} d }{\Delta_*}\left(\log\left(\frac{\Delta_* T}{CL}\right) + \sqrt{\log\left(\frac{\Delta_* T}{CL}\right)} \right).
    \]
\end{proof}

\section{Proof of Lemma~\ref{lem:upper_stability_sup2}}
\label{appdx:lem:upper_stability_sup2}
To show Lemma~\ref{lem:upper_stability_sup2}, we show the following lemmas. The proofs are shown in Appendices~\ref{appdx:lem:upper_stability_sup2} and \ref{appdx:lem:upper_stability_sup3}, respectively. 

\begin{lemma}
\label{lem:upper_stability_sup1}
    For $p, q\in[0, 1]$ and $\ell \geq -\frac{1-\alpha}{2}q^{\alpha-1}$, we have
    \begin{align}
    \label{eq:objective_lemma}
        \ell \cdot  \big(q - p\big)  - d\left(p, q\right) \leq \frac{2q^{2-\alpha}\ell^2}{1-\alpha}.
    \end{align}
\end{lemma}
\begin{lemma}
    \label{lem:upper_stability_sup3}
    Fix 
    $a^\dagger \in [K]$ and $q\in\mathcal{P}_K$. Then, the following holds for $\ell \in \mathbb{R}^K$:
    \begin{itemize}
        \item If $\ell(a) \geq - \frac{1 - \alpha}{4}q(a)^{\alpha-1}$ for all $a \in [K]$,
        we then have
        \begin{align}
            \label{eq:lem_1}
            \Big\langle \ell, q - p \Big\rangle - D(p, q) \leq \frac{4}{1-\alpha} \left(\sum_{a\in[K]}q(a)^{2-\alpha}\ell_t(a)^2\right)
        \end{align}
        for all $p \in\mathcal{P}_K$. 
        \item If $\ell(a) \geq - \frac{1 - \alpha}{4}q(a)^{\alpha-1}$ for all $a \in [K] \setminus \{a^\dagger\}$
        and $\ell(a^{\dagger}) \leq \frac{1 - \alpha}{4}(1-q(a^{\dagger}))^{\alpha-1}$,
        we then have
        \begin{align}
            \label{eq:lem_2}
            &\Big\langle \ell, q - p \Big\rangle - D(p, q)
            \leq \frac{4}{1-\alpha} \left(\sum_{a\neq a^\dagger}q(a)^{2-\alpha}\ell(a)^2 + \big( 1- q(a^\dagger)\big)^{2-\alpha}\ell(a^\dagger)^2\right)
        \end{align}
        for all $p \in\mathcal{P}_K$. 
    \end{itemize}
\end{lemma}
\begin{lemma}
    \label{lem:upper_stability_sup4}
    Fix 
    $a^\dagger \in [K]$ and $q\in\mathcal{P}_K$.
    If $|\ell(a)| \leq \frac{1 - \alpha}{4}q(a)^{\alpha-1}$ for all $a \in [K]$,
    we have
    \begin{align}
        \label{eq:sup4}
        \Big\langle \ell, q - p \Big\rangle - D(p, q)
        \leq \frac{4}{1-\alpha} \left(\sum_{a\neq a^\dagger}q(a)^{2-\alpha}\ell(a)^2 + \big( \min\left\{ q(a^\dagger), 1- q(a^\dagger) \right\}\big)^{2-\alpha}\ell(a^\dagger)^2\right)
    \end{align}
    for all $p \in\mathcal{P}_K$. 
\end{lemma}

Then, by using Lemma~\ref{lem:upper_stability_sup4},
we prove Lemma~\ref{lem:upper_stability_sup2} as follows:
\begin{proof}[Proof of Lemma~\ref{lem:upper_stability_sup2}]
    From Lemma~\ref{lem:upper_stability_sup3}
    with $\ell=\eta_t \widehat{\ell}_t(x)$ and $p = p(x)$,
    we have
    we have
    \begin{align*}
        &\left\langle \eta_t \widehat{\ell}_t(x), q(x) - p(x) \right\rangle - D(p(x), q(x))\\
        &\leq \frac{4\eta_t^2}{1 - \alpha}\left(\sum_{a \neq a^\dagger(x)}q(a\mid x)^{2-\alpha}\ell_t(a\mid x)^2 + \big(1-q(a^\dagger(x)\mid x)\big)^{2-\alpha}\ell_t(a^\dagger(x), x)^2\right)\\
        &\leq \frac{4 \eta_t^2}{1 - \alpha}\left(\sum_{a \neq a^\dagger(x)}q(a\mid x)^{2-\alpha}\ell_t(a\mid x)^2 + \min\big\{q(a^\dagger(x)\mid x), 1 - q(a^\dagger(x)\mid x)\big\}^{2-\alpha}\ell_t\left(a^\dagger(x), x\right)^2\right).
    \end{align*}
    By dividing both sides by $\eta_t$,
    we obtain the desired bound.
\end{proof}

\subsection{Proof of Lemma~\ref{lem:upper_stability_sup1}}
\label{appdx:lem:upper_stability_sup1}
\begin{proof}
For all given $q$ and $\ell$, the LHS of \eqref{eq:objective_lemma} is concave in $p$. Hence, this is maximized when 
\begin{align}
\label{eq:first_prep}
    &\frac{\mathrm{d}}{\mathrm{d}p}\Big\{\ell\cdot \big(q - p\big) - d\left(p, q\right)\Big\} = -\ell - q^{\alpha-1} + p^{\alpha-1} = 0.
\end{align}
We then have
\begin{align}
\label{eq:third_prep}
    p = \big( q^{\alpha-1} + \ell \big)^{\frac{1}{\alpha-1}} \leq \left(q^{\alpha-1} - \frac{1-\alpha}{2}q^{\alpha-1}\right)^{\frac{1}{\alpha-1}} = q\left(1 - \frac{1-\alpha}{2}\right)^{\frac{1}{\alpha - 1}} \leq 2q,
\end{align}
where the first equality follows from \eqref{eq:first_prep} and the first inequality follows from the assumption of $\ell \geq - \frac{1-\alpha}{2}q^{\alpha-1}$. Furthermore, from the intermediate value theorem and the fact that $p^{\alpha-2}$ is monotone decreasing in $p$, we have
\begin{align*}
    \big|\ell\big| &= \big|p^{\alpha-1} - q^{\alpha-1}\big|\\
    &\geq \min\Big\{\big| (\alpha-1)p^{\alpha-2} \big|, \big| (\alpha - 1)q^{\alpha-2} \big| \Big\}\big| p - q \big|\\
    &= (1-\alpha)\max\big\{p, q\big\}^{\alpha-2}\big|p-q\big|,
\end{align*}
where the first inequality follows from \eqref{eq:first_prep} and the second inequality follows from the intermediate value theorem. This implies 
\begin{align}
\label{eq:second_prep}
    \big| p - q \big| \leq \frac{1}{1-\alpha}\cdot \max\big\{p, q\big\}^{2-\alpha}\big|\ell\big|.
\end{align}
We then have
\begin{align*}
    &\ell\cdot (p - q) - d(p, q) \leq \big|\ell\big| \big|q - p\big|\leq \frac{\ell^2}{1-\alpha}\max\big\{p, q\big\}^{2-\alpha} \leq \frac{4\ell^2}{1-\alpha} q^{2-\alpha},
\end{align*}
where the second inequality follows from \eqref{eq:second_prep} and the last inequality follows from \eqref{eq:third_prep}. 
\end{proof}

\subsection{Proof of Lemma~\ref{lem:upper_stability_sup3}}
\label{appdx:lem:upper_stability_sup3}
\begin{proof}
    We have
    \begin{align}
        \nonumber
        &\Big\langle \ell, q - p \Big\rangle - D(p, q)\\
        \nonumber
        &= \frac{1}{2} \sum_{a\neq a^\dagger(x)}\left( 2\ell(a ) \cdot \big(q(a) - p(a)\big) - d\big(p(a), q(a)\big)\right) \\
        \nonumber
        &\quad \quad \frac{1}{2}\Bigg(2\ell_t(a^\dagger)\cdot \big(q(a^\dagger) - p(a^\dagger)\big) - d\big(p(a^\dagger), q(a^\dagger)\big)
        - \sum_{a \in [K]} d\big(p(a), q(a)\big)\Bigg)\\
        \nonumber
        &\leq \frac{1}{2} \sum_{a\neq a^\dagger}\Big\{2\ell(a) \cdot \big(q(a) - p(a)\big) - d\big(p(a), q(a)\big)\Big\}\\
        \nonumber
        &\ \ \ \ \ + \frac{1}{2}\min\Bigg\{ 2\ell(a^\dagger)\cdot \big(q(a^\dagger) - p(a^\dagger)\big) - d\big(p(a^\dagger), q(a^\dagger)\big),
        \nonumber
        \\ &
        \quad \quad \quad \quad \quad \quad
        2\ell(a^\dagger) \cdot \big(q(a^\dagger) - p(a^\dagger)\big) - \sum_{a\neq a^\dagger} d\big(p(a), q(a)\big)\Bigg\}.
        \label{eq:sup3_01}
    \end{align}

    From Lemma~\ref{lem:upper_stability_sup1}, if $\ell(a) \geq -\frac{1-\alpha}{2}q(a)^{\alpha-1}$, we have
    \begin{align}
        2\ell(a)\cdot \left(q(a) - p(a)\right) - 
        d\left(p(a), q(a)\right) \leq \frac{8q(a)^{2-\alpha}\ell(a)^2}{1-\alpha}.
        \label{eq:sup3_02}
    \end{align}

    Furthermore, we have
    \begin{align*}
        q(a^\dagger) - p(a^\dagger) 
        = \big(1  - p(a^\dagger) \big) - \big(1 - q(a^\dagger ) \big) = \sum_{a\neq a^\dagger}\big(p(a) - q(a) \big).
    \end{align*}

    As we have 
    $\big(1 - q(a^\dagger)\big)^{\alpha-1}\leq q(a)^{\alpha-1}$ for all $a\in[K]\backslash\{a^\dagger\}$,
    if $\ell(a^\dagger)\leq \frac{1-\alpha}{4}\big(1 - q(a^\dagger)\big)^{\alpha-1}$,
    we then have
    \[
        -\ell(a^\dagger)\geq - \frac{1-\alpha}{4}q(a)^{\alpha-1}
    \]
    for all $a\in[K] \setminus \{a^\dagger(x)\}$.
    Hence,
    Lemma~\ref{lem:upper_stability_sup1} yields
    \begin{align*}
        &2\ell(a^\dagger)\cdot \big(q(a^\dagger) - p(a^\dagger)\big) - \sum_{a\neq a^\dagger} d\big(p(a), q(a)\big)\\
        &= \sum_{a\neq a^\dagger}\left( -2\ell(a^\dagger) \cdot \big(q(a) - p(a)\big) - d\big(p(a), q(a)\big) \right)\\
        &\leq \frac{2}{1-\alpha} \sum_{a\neq a^\dagger}\left(2\ell(a^\dagger)\right)^2q(a)^{2-\alpha}\\
        &\leq \frac{8}{1-\alpha} \ell(a^\dagger)^2 \left(\sum_{a\neq a^\dagger}q(a)\right)^{2-\alpha}\\
        &= \frac{8}{1-\alpha}\big( 1 - q(a^\dagger) \big)^{2-\alpha}(\ell(a^\dagger))^2 
    \end{align*}
    if $\ell(a^{\dagger}) \le \frac{1-\alpha}{4} \left( 1 - q(a^{\dagger} )\right)^{\alpha - 1}$.
    Combining this with \eqref{eq:sup3_01} and \eqref{eq:sup3_02},
    we obtain the desired bounds.
\end{proof}

\subsection{Proof of Lemma~\ref{lem:upper_stability_sup4}}
\label{appdx:lem:upper_stability_sup4}
\begin{proof}
    Let us consider two cases:
    when
    $q(a^{\dagger}) \le 1 - q(a^{\dagger})$
    and
    when
    $q(a^{\dagger}) < 1 - q(a^{\dagger})$.
    Suppose
    $q(a^{\dagger}) \le 1 - q(a^{\dagger})$.
    Then,
    from \eqref{eq:lem_1} in Lemma~\ref{lem:upper_stability_sup3},
    we have
    \begin{align*}
        \Big\langle \ell, q - p \Big\rangle - D(p, q)
        &
        \leq 
        \frac{4}{1-\alpha} \left(\sum_{a\neq a^\dagger}q(a)^{2-\alpha}\ell(a)^2 + \big(  q(a^\dagger) \big)^{2-\alpha}\ell(a^\dagger)^2\right)
        \\
        &
        \leq 
        \frac{4}{1-\alpha} \left(\sum_{a\neq a^\dagger}q(a)^{2-\alpha}\ell(a)^2 + \big( \min\left\{ q(a^\dagger), 1- q(a^\dagger) \right\}\big)^{2-\alpha}\ell(a^\dagger)^2\right).
    \end{align*}
    Suppose
    $q(a^{\dagger}) > 1 - q(a^{\dagger})$.
    We then have
    $
        |\ell(a^{\dagger}) |
        \le 
        \frac{1-\alpha}{4} (q(a^{\dagger}))^{\alpha - 1}
        <
        \frac{1-\alpha}{4} (1 - q(a^{\dagger}))^{\alpha - 1}
    $.
    Hence,
    we can use \eqref{eq:lem_2} in Lemma~\ref{lem:upper_stability_sup3} to obtain:
    \begin{align*}
        \Big\langle \ell, q - p \Big\rangle - D(p, q)
        &
        \leq 
        \frac{4}{1-\alpha} \left(\sum_{a\neq a^\dagger}q(a)^{2-\alpha}\ell(a)^2 + \big( 1 - q(a^\dagger) \big)^{2-\alpha}\ell(a^\dagger)^2\right)
        \\
        &
        \leq 
        \frac{4}{1-\alpha} \left(\sum_{a\neq a^\dagger}q(a)^{2-\alpha}\ell(a)^2 + \big( \min\left\{ q(a^\dagger), 1- q(a^\dagger) \right\}\big)^{2-\alpha}\ell(a^\dagger)^2\right),
    \end{align*}
    which completes the proof. 
\end{proof}

\section{The BoBW-RealFTRL with arm-dependent features}
\label{sec:bobw_linearftrl}
\colorred
In this section, we reformulate the FTRL with Shannon entropy under the setting with arm-dependent features and derive its regret upper bound. Note that \citet{kuroki2023bestofbothworlds} and \citet{kato2023bestofbothworlds} investigate the FTRL with Shannon entropy under the setting with arm-independent features.
\colorblack

We consider the same formulation in our main text. However, we do not assume the finiteness of contexts in a stochastic regime. 

\subsection{The BoBW-RealFTRL}
Following \citet{kuroki2023bestofbothworlds} and \citet{kato2023bestofbothworlds}, we define the BoBW-RealFTRL with the arm-dependent feature as
\begin{align}
\label{eq:policy2}
    \pi_t(X_t) \coloneqq (1-\gamma_t)q_t(X_t) + {\gamma_t}e^*(X_t),
\end{align}
where  
\begin{align*}
    &q_t(x) \in \argmin_{q\in \mathcal{P}_K}\left\{ \sum^{t-1}_{s=1}\left\langle \widehat{\ell}_s(x), q\right\rangle + \frac{1}{\eta_t}\psi(q)\right\}\ \ \mathrm{for}\ t \geq 2,\\
    &q_1(x) \coloneqq (1/K\ 1/K\ \cdots\ 1/K)^\top,\nonumber\\
    &\psi(q(x)) \coloneqq - \sum_{a\in[K]}q(a\mid x)\log \left(\frac{1}{q(a\mid x)}\right),\\
    &\frac{1}{\eta_{t+1}}  \coloneqq \frac{1}{\eta_t}  + \frac{1}{\eta_1} \frac{1}{\sqrt{1 + \big(\log(K)\big)^{-1}\sum^t_{s=1}H\big(q_s(X_s)\big)}},\\
    &\frac{1}{\eta_1} \coloneqq \sqrt{\frac{\left(\frac{1}{ \lambda} + d\right)\log(T)}{\log(K)}},\quad \mathrm{and}\quad\gamma_t \coloneqq \frac{\eta_t}{ \lambda}.
\end{align*}

\begin{algorithm}[tb]
    \caption{BoBW-RealFTRL.}
    \label{alg_realftrl}
    \begin{algorithmic}
    \STATE {\bfseries Parameter:} Learning rate $\eta_1, \eta_2,\dots, \eta_T > 0$. 
    \STATE {\bfseries Initialization:} Set $\theta_{0} = 0$.
    \FOR{$t=1,\dots, T$}
    \STATE Observe $X_t$. 
    \STATE Draw $A_t \in [K]$ following the policy $\pi_t(X_t) \coloneqq (1-\gamma_t)q_t(X_t) + \gamma_te^*(X_t)$ defined in \eqref{eq:policy2}.
    \STATE Observe the loss $\ell_t(A_t, X_t)$.
    \STATE Compute $\widehat{\theta}_t$. 
    \ENDFOR
\end{algorithmic}
\end{algorithm}

\subsection{Regret analysis}
\label{sec:regretanalysis_realftrl}
This section provides upper bounds for the regret of our proposed BoBW-RealFTRL algorithm. 

To derive upper bounds, we define the following quantities:
\begin{align*}
    &Q(\rho^*_T\mid x) = \sum^T_{t=1}\Big\{ 1 - q_t\big(\rho^*_T(x)\mid x\big) \Big\},\\
    &\overline{Q}(\rho^*_T) = \mathbb{E}\left[ Q(\rho^*_T\mid X_0) \right].
\end{align*}

Then, we show the following upper bound, which holds for general cases such as adversarial and stochastic regimes. We show the proof in Appendix~\ref{appdx:regret_bound_ftrl}.
\begin{theorem}[General regret bounds]
\label{thm:regret_bound_realftrl}
Consider the BoBW-RealFTRL. Assumptions~\ref{asm:contextual_dist}--\ref{asm:bounded_random} hold. 
Then, the decision-maker incurs the following regret:
    \begin{align*}
    R_T &= O\left(\left( \eta_1\left(\frac{1}{ \lambda} + d\right)\frac{\log(T)}{\sqrt{\log(K)}} + \frac{1}{\eta_1}\sqrt{\log(K)}\right) \sqrt{\log(KT)}\max\Big\{\overline{Q}^{1/2}(\rho^*_T), 1\Big\}\right).
\end{align*}
\end{theorem}

For each regime, we derive a specific upper bound. The proof is shown in Appendix~\ref{appdx:adversarial_regime}. First, from $\overline{Q}(\rho^*_T) \leq T$, the following regret bound holds without any assumptions on the loss; that is, it holds in an adversarial regime. 
\begin{theorem}[Regret upper bound in an adversarial regime]
    Consider the BoBW-RealFTRL.
    Assume that the loss is generated under an adversarial regime. Suppose that Assumption~\ref{asm:contextual_dist} and \ref{asm:contextual_dist2} hold. Then, the regret satisfies 
    \[R_T = {O}\left(\log(KT)\sqrt{ \left(\frac{1}{ \lambda} + d\right)\log(T)T}\right).\]
\end{theorem} 

Next, we derive a regret upper bound in a stochastic regime with a margin condition.
\begin{theorem}[Regret upper bound in a stochastic regime with a margin condition]
\label{thm:ftrl_shanon_margin}
    Consider the BoBW-RealFTRL.
Assume that the loss is generated under a stochastic regime with a margin condition (Definition~\ref{def:margin}). Suppose that Assumption~\ref{asm:contextual_dist}, \ref{asm:contextual_dist2}--\ref{asm:bounded_random} hold. Then, the regret satisfies 
\[R_T = O\left(\left(\frac{1 + \beta}{\beta\Delta_*}\right)^{\frac{\beta}{2+\beta}}\left(\log(KT)\sqrt{ {\left(\frac{1}{ \lambda} + d\right)\log(T)}}\right)^{\frac{1+\beta}{2+\beta}}T^{\frac{1}{2+\beta}}\right).\] 
\end{theorem}
We omit the proof because it is almost the same as that for Theorem~\ref{thm:regret_upper_margin} for the $1/2$-LC-Tsallis-INF

Furthermore, we derive a regret bound under the linear contextual adversarial regime with a self-bounding constraint. The proof is provided in Appendix~\ref{appdx:adversarial_regime}
\begin{theorem}[Regret bounds under the linear contextual adversarial regime with a self-bounding constraint]
\label{cor:adversarial_regime}
Consider the BoBW-RealFTRL.
Assume that the loss is generated under a linear contextual adversarial regime with a self-bounding constraint (Definition~\ref{def:bounding}). Suppose that Assumption~\ref{asm:contextual_dist}, \ref{asm:contextual_dist2}--\ref{asm:bounded_random} hold. Then, the regret satisfies
\[
R_T = O\left(\frac{\left(\frac{1}{ \lambda} + d\right)\log(T)\log(KT)}{\Delta_{*}} + \sqrt{\frac{C\left(\frac{1}{ \lambda} + d\right)\log(T)\log(KT)}{\Delta_{*}}}\right).\] 
\end{theorem}

Note that the BoBW-RealFTRL does not require Assumption~\ref{asm:contextual_dist_stochastic} in stochastic regimes.

\section{Proof of Theorem~\ref{thm:regret_bound_realftrl}}
\label{appdx:regret_bound_ftrl}
This section provides the proof of Theorems~\ref{thm:regret_bound_realftrl}. 

As well as Section~\ref{appdx:stab_penal}, based on the result in Proposition~\ref{lem:basic}, we define 
\begin{align*}
    &r_T(x) \coloneqq\\
    &\sum^T_{t=1}\left( \gamma_t + \left\langle \widehat{\ell}_t(x), q_t(x) - q_{t+1}(x) \right\rangle- \frac{1}{\eta_t}D(q_{t+1}(x), q_t(x)) \right. 
    \left. - \left(\frac{1}{\eta_{t}} - \frac{1}{\eta_{t-1}}\right)\psi\big(q_t(x)\big) \right),
\end{align*}
where we replace $H(q(x))$ with $\psi(q(x))$. 
Note that we define $1/\eta_{0} = 0$ as an exception. Using this function $r_T(x) $, we can bound $R_T$ as 
$
    R_T \leq \mathbb{E}\big[r_T(X_0)\big].
$

We decompose the pointwise regret upper bound $r_T(x)$ as follows: 
\begin{align*}
    r_T(x) 
    &= \annot{
        \sum^T_{t=1} \left( \gamma_t + \left\langle \widehat{\ell}_t(x), q_t(x) - q_{t+1}(x) \right\rangle - \frac{1}{\eta_t}D(q_{t+1}(x), q_t(x)) \right) 
        }{=\ stability term}
    \nonumber\\
    &\ \ \ \ \ \ \ \ \ \ \ \ \ \ \ \ \ \ \ \ \ 
    +
    \annot{
            \left(
            - 
            \sum^T_{t=1}\left(\frac{1}{\eta_{t}} - \frac{1}{\eta_{t-1}}\right)\psi\big(q_t(x)\big)
            \right)
    }{=\ penalty term}. 
\end{align*}

To bound the stability term $\left\langle \widehat{\ell}_t(x), q_t(x) - q_{t+1}(x) \right\rangle - \frac{1}{\eta_t}D(q_{t+1}(x), q_t(x))$, we use the following proposition from \citet{ito2022nearly}.
\begin{proposition}[From Lemma~8 in \citet{ito2022nearly}]
\label{prp:lem8}
   Consider the BoBW-RealFTRL. For all $x\in\mathcal{X}$, all $\ell: \mathcal{X} \to \mathbb{R}^K$ and $p, q \in\Pi$, it holds that
    \begin{align*}
        &\Big\langle \ell_t(x), p(x) - q(x) \Big\rangle - \frac{1}{\eta_t}D(q(x), p(x)) \leq \frac{1}{\eta_t}\sum_{a\in[K]}p(a\mid x)\xi\Big(\eta_t{\ell_t(a, x)}\Big),
    \end{align*}
    where $\xi(x) \coloneqq \exp(-x) + x - 1$. 
\end{proposition}

By using Proposition~\ref{prp:lem8}, we obtain the following lemma. 
\begin{lemma}
\label{thm:basic_bounds}
The regret for the BoBW-RealFTRL satisfies
    \begin{align*}
        R_T &\leq \mathbb{E}\left[\sum^T_{t=1}\left(\gamma_t + 3\eta_td + \left(
            - 
            \sum^T_{t=1}\left(\frac{1}{\eta_t} - \frac{1}{\eta_{t-1}}\right)\psi\big(q_t(X_0)\big)
            \right)\right)\right].
    \end{align*}
\end{lemma}
\begin{proof}[Proof of Lemma~\ref{thm:basic_bounds}]
From Proposition~\ref{lem:basic}, we have
\begin{align*}
    R_T \leq \mathbb{E}\big[r_T(X_0)\big],
\end{align*}
where 
\begin{align*}
    r_T(x) 
    &= \annot{
        \sum^T_{t=1} \left( \gamma_t + \left\langle \widehat{\ell}_t(x), q_t(x) - q_{t+1}(x) \right\rangle - \frac{1}{\eta_t}D(q_{t+1}(x), q_t(x)) \right) 
        }{=\ stability term}
    \nonumber\\
    &\ \ \ \ \ \ \ \ \ \ \ \ \ \ \ \ \ \ \ \ \ 
    +
    \annot{
            \left(
            - 
            \sum^T_{t=1}\left(\frac{1}{\eta_t} - \frac{1}{\eta_{t-1}}\right)\psi\big(q_t(x)\big)
            \right)
    }{=\ penalty term}.
\end{align*}

To prove the statement, we bound the stability term as 
\begin{align}
\label{eq:target_lemma49}
    &\mathbb{E}\left[\left\langle \widehat{\ell}_t(X_0), \pi_t(X_0) - q_{t+1}(X_0) \right\rangle - D_t(q_{t+1}(X_0), \pi_t(X_0))\right] \leq  3\eta_td. 
\end{align}

To show this, from Proposition~\ref{prp:lem8}, we have
\begin{align*}
    &\left\langle \widehat{\ell}_t(x), \pi_t(x) - q_{t+1}(x) \right\rangle - D_t(q_{t+1}(x), \pi_t(x))\leq \frac{1}{\eta_t}\sum_{a\in[K]}\pi_t(a\mid x)\xi\left(\eta_t{\widehat{\ell}_t(a, x)}\right).
\end{align*}

Here, from Lemma~\ref{lem:upper_loss}, we have $\eta_t \widehat{\ell}_t(a, x) \geq -\eta_t/ \left(\lambda \gamma_t \right)$. Additionally, since $\gamma_t = \frac{\eta_t}{ \lambda}$, $\eta_t{\widehat{\ell}_t(a, x)} = -1$ holds. Then, we have
\begin{align*}
    &\left\langle \widehat{\ell}_t(x), \pi_t(x) - q_{t+1}(x) \right\rangle - D_t(q_{t+1}(x), \pi_t(x))\\
    &\leq \frac{1}{\eta_t}\sum_{a\in[K]}\pi_t(a\mid x)\xi\left(\eta_t{\widehat{\ell}_t(a, x)}\right)\\
    &\leq  \eta_t\sum_{a\in[K]}\pi_t(a\mid x)\widehat{\ell}^2_t(a, x). 
\end{align*}

Lastly, from Lemma~\ref{prp:lemma6_neu}, which states that 
$   \mathbb{E}\left[\sum_{a\in[K]}\pi_t(a\mid X_0)\widehat{\ell}_t(a, X_0)^2\mid \mathcal{F}_{t-1}\right] \leq d
$, we have \eqref{eq:target_lemma49}. 
\end{proof}

From this result, we obtain the following lemma. 
\begin{lemma}
\label{lem:regret_basic_bounds}
    Assume the conditions in Theorem~\ref{thm:basic_bounds}. Consider the BoBW-RealFTRL. Then, we have
    \begin{align*}
        R_T \leq \overline{c} \sqrt{\mathbb{E}\left[\sum^T_{t=1}\psi(q_t(X_0))\right]},
    \end{align*}
    where $\overline{c} = O\left( \eta_1\left(\frac{1}{ \lambda} + d\right)\frac{\log(T)}{\sqrt{\log(K)}} + \frac{1}{\eta_1}\sqrt{\log(K)}\right)$. 
\end{lemma}
\begin{proof}
We prove the following inequalities:
\begin{align}
\label{eq:target1}
    \sum^T_{t=1}\left(\gamma_t + 3\eta_t d\right)  &= O\left(\eta_1\left(\frac{1}{\lambda} + d\right)\frac{\log(T)}{\sqrt{\log(K)}}\sqrt{\sum^T_{t=1}\psi\big(q_t(X_t)\big)}\right),\\
     \label{eq:target2}
    - \sum^T_{t=1}\left(\frac{1}{\eta_t} - \frac{1}{\eta_{t-1}}\right)\psi\big(q_t(x)\big) &= O\left( \beta_1\sqrt{\log(K)} \sqrt{\sum^T_{t=1}H(q_t(X_{t}))} \right).
\end{align}

\paragraph{Proof of \eqref{eq:target1}} From $\gamma_t = \frac{\eta_t}{ \lambda}$, it holds that
\begin{align*}
\sum^T_{t=1}\left(\gamma_t + 3\eta_t d\right) = \sum^T_{t=1}\left( \frac{\eta_t}{ \lambda} + 3\eta_t d\right) = \left(\frac{1}{\lambda} + {3d}\right)\sum^T_{t=1}\eta_t.
\end{align*}
From $\frac{1}{\eta_{t+1}}  \coloneqq \frac{1}{\eta_t}  + \frac{1}{\eta_1} \frac{1}{\sqrt{1 + \big(\log(K)\big)^{-1}\sum^t_{s=1}H\big(q_s(X_s)\big)}}$, we have 
\begin{align*}
    \frac{1}{\eta_{t}} &= \frac{1}{\eta_{1}} + \sum^{t-1}_{u=1}\frac{1}{\eta_{1}}\frac{1}{\sqrt{1 + \big(\log(K)\big)^{-1}\sum^u_{s=1}\psi\big(q_{s}(X_s)\big)}}\\
    &\geq \frac{1}{\eta_1}\frac{t}{\sqrt{1 + \big(\log(K)\big)^{-1}\sum^t_{s=1}\psi\big(q_{s}(X_s)\big)}}.
\end{align*}
Therefore, we have
\begin{align*}
    \sum^T_{t=1}\eta_t &\leq \sum^T_{t=1}\eta_1\frac{\sqrt{1 + \big(\log(K)\big)^{-1}\sum^t_{s=1}H\big(q_{s}(X_s)\big)}}{t}\\
    &\leq \eta_1\left(1 + \log(T)\right)\sqrt{1 + \big(\log(K)\big)^{-1}\sum^T_{t=1}\psi\big(q_t(X_t)\big)}.
\end{align*}
Therefore, we obtain
\begin{align*}
\sum^T_{t=1}\left(\gamma_t + 3\eta_t d\right) = O\left( \eta_1\left(\frac{1}{\lambda} + d\right)\frac{\log(T)}{\sqrt{\log(K)}}\sqrt{\sum^T_{t=1}\psi\big(q_t(X_t)\big)}\right).
\end{align*}

\paragraph{Proof of \eqref{eq:target2}} From the definitions of $\beta_t$ and $\gamma_t$, we have
\begin{align*}
    &- \sum^T_{t=1}\left(\frac{1}{\eta_t} - \frac{1}{\eta_{t-1}}\right)\psi\big(q_t(x)\big)\\
    &= \sum^T_{t=1}\frac{1}{\eta_1\sqrt{1 + \big(\log(K)\big)^{-1}\sum^{t-1}_{s=1}\psi\big(q_{s}(X_s)\big)}}\psi(q_t(X_t))\\
    &= 2\frac{1}{\eta_1}\sqrt{\log(K)}\sum^T_{t=1}\frac{\psi\big(q_t(X_t)\big)}{\sqrt{\log(K) + \sum^{t-1}_{s=1}\psi\big(q_{s}(X_s)\big)} + \sqrt{\log(K) + \sum^{t-1}_{s=1}\psi\big(q_{s}(X_s)\big)}}\\
    &\leq 2\frac{1}{\eta_1}\sqrt{\log(K)}\sum^T_{t=1}\frac{\psi\big(q_{t}(X_{t})\big)}{\sqrt{\log(K) + \sum^{t}_{s=1}\psi\big(q_{s}(X_s)\big)} + \sqrt{\log(K) + \sum^{t-1}_{s=1}\psi\big(q_{s}(X_s)\big)}}\\
    &\leq 2\frac{1}{\eta_1}\sqrt{\log(K)}\sum^T_{t=1}\frac{\psi\big(q_{t}(X_{t})\big)}{\sqrt{\sum^{t}_{s=1}\psi\big(q_{s}(X_s)\big)} + \sqrt{\sum^{t-1}_{s=1}\psi\big(q_{s}(X_s)\big)}}\\
    &= 2\frac{1}{\eta_1}\sqrt{\log(K)}\sum^T_{t=1}\frac{\psi(q_t(X_t))}{\psi(q_t(X_t))}\left\{{\sqrt{\sum^t_{s=1}\psi\big(q_{s}(X_s)\big)} - \sqrt{\sum^{t-1}_{s=1}\psi\big(q_{s}(X_s)\big)}}\right\}\\
    &= 2\frac{1}{\eta_1}\sqrt{\log(K)}\sum^T_{t=1}\left\{{\sqrt{\sum^t_{s=1}\psi\big(q_{s}(X_s)\big)} - \sqrt{\sum^{t-1}_{s=1}\psi\big(q_{s}(X_s)\big)}}\right\}\\
    &= 2\frac{1}{\eta_1}\sqrt{\log(K)}\left\{{\sqrt{\sum^{T}_{s=1}\psi\big(q_{s}(X_s)\big)} - \sqrt{\psi\big(q_{1}(X_1)\big)}}\right\}\\
    &\leq 2\frac{1}{\eta_1}\sqrt{\log(K)}\sqrt{\sum^{T}_{s=1}\psi\big(q_{s}(X_s)\big)}.
\end{align*}

Inequalities \eqref{eq:target1} and \eqref{eq:target2} combined with the inequality in Lemma~\ref{thm:basic_bounds} yield 
\begin{align*}
&R_T \leq \mathbb{E}\left[\sum^T_{t=1}\left(\gamma_t + 3\eta_td + \left(
            - 
            \sum^T_{t=1}\left(\frac{1}{\eta_t} - \frac{1}{\eta_{t-1}}\right)\psi\big(q_t(X_0)\big)
            \right)\right)\right]\\
    &= \mathbb{E}\left[\sum^T_{t=1}\left(\gamma_t + 3\eta_td + \left(
            - 
            \sum^T_{t=1}\left(\frac{1}{\eta_t} - \frac{1}{\eta_{t-1}}\right)\psi\big(q_t(X_t)\big)
            \right)\right)\right]\\
    &= \mathbb{E}\left[O\left( \eta_1\left(\frac{1}{\lambda} + d\right)\frac{\log(T)}{\sqrt{\log(K)}}\sqrt{\sum^T_{t=1}\psi\big(q_t(X_t)\big)}\right) + O\left( \frac{1}{\eta_1}\sqrt{\log(K)} \sqrt{\sum^T_{t=1}\psi(q_t(X_{t})} \right)\right]\\
    &= O\left( \eta_1\left(\frac{1}{\lambda} + d\right)\frac{\log(T)}{\sqrt{\log(K)}}\sqrt{\sum^T_{t=1}\mathbb{E}\left[\psi\big(q_t(X_t)\big)\right]}\right) + O\left( \frac{1}{\eta_1}\sqrt{\log(K)} \sqrt{\sum^T_{t=1}\mathbb{E}\left[\psi(q_t(X_{t})\right]} \right).
\end{align*}
Thus, we obtain the regret bound in Lemma~\ref{lem:regret_basic_bounds}.
\end{proof}

Next, we consider bounding $\sum^T_{t=1}\psi(q_t(x))$ by $Q(\rho^*_T\mid x)$ as shown in the following proposition.
\begin{proposition}
[From Lemma~4 in \citet{ito2022nearly}]
\label{lem:bound_H}
For all $\rho^*:\mathcal{X}\to[K]$, the following holds:
\begin{align*}
    \sum^T_{t=1}\psi(q_t(x)) \leq Q(\rho^*\mid x) \log\left(\frac{eKT}{Q(\rho^*\mid x)}\right),
\end{align*}
where $e$ is Napier's constant. 
\end{proposition}

By using the above lemmas and propositions, we prove Theorem~\ref{thm:regret_bound_realftrl}.
\begin{proof}[Proof of Theorem~\ref{thm:regret_bound_realftrl}]
    From Lemma~\ref{lem:bound_H}, if $Q(\rho^*_T \mid x) \leq e$, we have $\sum^T_{t=1}\psi(q_t(x)) \leq e\log(KT)$ and otherwise, we have $\sum^T_{t=1}\psi(q_t(x)) \leq Q(\rho^*_T \mid x)\log(KT)$. Hence, we have $\sum^T_{t=1}\psi(q_t(x))  \leq \log(KT)\max\{e, Q(\rho^*_T \mid x)\}$. From Lemma~\ref{lem:regret_basic_bounds}, we have
    \begin{align*}
        R_T &\leq \overline{c} \sqrt{\sum^T_{t=1}\mathbb{E}\left[\psi(q_t(X_0))\right]}\\
        &= O\left(\left( \eta_1\left(\frac{1}{ \lambda} + d\right)\frac{\log(T)}{\sqrt{\log(K)}} + \frac{1}{\eta_1}\sqrt{\log(K)}\right)\sqrt{\log(KT)}\max\Big\{\overline{Q}^{1/2}, 1\Big\}\right).
    \end{align*}
\end{proof}

\section{Proof of Theorem~\ref{cor:adversarial_regime}}
\label{appdx:adversarial_regime}
\begin{proof}[Proof of Theorem~\ref{cor:adversarial_regime}]
    From the definition of the contextual adversarial regime with a self-bounding constraint, we have
    \begin{align*}
        R_T &\geq \Delta_* \cdot \mathbb{E}\left[\sum^T_{t=1}\Big( 1 - \pi_t(a^*(X_0)\mid X_0) \Big)\right] - C\\
        &=  \Delta_* \cdot \overline{Q}(\rho^*_T) - C.
    \end{align*}
    Therefore, from Lemma~\ref{lem:regret_basic_bounds}, for all $\lambda > 0$, we have
    \begin{align*}
        R_T &= (1+\lambda)R_T - \lambda R_T\\
        &= (1+\lambda)O\left(\overline{c}\sqrt{\log(KT)}\sqrt{\sum^T_{t=1}\mathbb{E}\left[\psi(q_t(X_0))\right]}\right) - \lambda R_T\\
        &\leq (1+\lambda)O\left(\overline{c}\sqrt{\log(KT)}\sqrt{\sum^T_{t=1}\mathbb{E}\left[\psi(q_t(X_0))\right]}\right) - \lambda \Delta_* \cdot \overline{Q}(\rho^*_T) + \lambda C,
    \end{align*}
    where 
    \[\overline{c} = \eta_1\left(\frac{1}{ \lambda} + d\right)\frac{\log(T)}{\sqrt{\log(K)}} + \frac{1}{\eta_1}\sqrt{\log(K)}.\]
    Here, as well as the proof of Theorem~\ref{thm:regret_bound}, from Lemma~\ref{lem:bound_H}, if 
    \[Q(\rho^*_T(x)\mid x) \leq e,\] 
    we have 
    \[\sum^T_{t=1}\psi(q_t(x)) \leq e\log(KT)\] and otherwise, we have \[\sum^T_{t=1}\psi(q_t(x)) \leq Q(\rho^*_T\mid x)\log(KT).\] Hence, we have $\sum^T_{t=1}\psi(q_t(x))  \leq \log(KT)\max\{e, Q(\rho^*_T\mid x)\}$. Here, to upper bound $R_T$, it is enough to only consider a case with $Q(\rho^*_T\mid x) \geq e$, and we obtain
    \begin{align*}
        R_T &\leq (1+\lambda)O\left(\overline{c}\sqrt{\log(KT)}\sqrt{\overline{Q}(\rho^*_T)\log(KT)}\right) - \lambda \Delta_* \cdot \overline{Q}(\rho^*_T) + \lambda C\\
        &\leq \frac{O\left(\Big\{(1+\lambda)\overline{c}\Big\}^2\sqrt{\log(KT)}\right)}{2\lambda\Delta_*} + \lambda \Delta_*.
    \end{align*}
    where the second inequality follows from $a\sqrt{b} - \frac{c}{2}b \leq \frac{a^2}{c^2}$ holds for all $a,b,c > 0$. By choosing
    \[\lambda = \sqrt{\frac{\overline{c}^2\log(KT)}{\Delta_*}\Big/ \left(\frac{c^2\log(KT)}{\Delta_*} + 2C\right)}.\]
    Then, we obtain $R_T = O\Bigg(\overline{c}^2\log(KT)/ \Delta_{*} + \sqrt{{C\overline{c}^2}\log(KT)/{\Delta_{*}}}\Bigg)$. 
\end{proof}

\section{Details of regret transformation}
\label{appdx:transform}

This section provides the details of regret transformation in Section~\ref{sec:regret_transform}. 

We consider two problems of linear contextual bandits: P.Indep $(\widetilde{d}, \widetilde{K}, \widetilde{T})$ and P.Dep $(d, K, T)$, defined below. 

\paragraph{P.Indep $(\widetilde{d}, \widetilde{K}, \widetilde{T})$} In this stetting, there are $ \widetilde{K}$ arms and $ \widetilde{T}$ rounds. We can observe $\widetilde{d}$-dimensional contexts. Let $\widetilde{\phi}: \mathcal{X} \to \mathbb{R}^d$ be a context map. We consider a linear model $\widetilde{\ell}_t(a, X_t) = \left\langle \widetilde{\theta}_{a, t}, \widetilde{\phi}(X_t) \right\rangle  + \widetilde{\varepsilon}_t(a)$, where $\widetilde{\theta}_{a, t}\in\mathbb{R}^{\widetilde{d}}$ is a regression coefficient, and $\widetilde{\varepsilon}_t(a)$ is the error term. Then, in each round $t$, given $\left\{\widetilde{\theta}_{a, t}\right\}_{a\in[K]}$
\begin{itemize}
    \item We observe $\widetilde{d}$-dimensional contexts $\widetilde{\phi}(X_t)$.
    \item The loss is generated from a linear model $\widetilde{\ell}_t(a, X_t) = \left\langle \theta_{a, t}, \widetilde{\phi}(X_t) \right\rangle  + \widetilde{\varepsilon}_t(a)$.
\end{itemize}
 
This problem can be transformed into P.Dep $(\widetilde{d}\widetilde{K}, \widetilde{K}, \widetilde{T})$ by defining the corresponding feature and regression parameters well. 

Let us define 
\[\phi(a, X_t) = \begin{pmatrix}
    \bm{0}_d\\
    \vdots\\
    \widetilde{\phi}(X_t)\\
    \vdots\\
    \bm{0}_d
\end{pmatrix},\quad 
\theta_t = \begin{pmatrix} \widetilde{\theta}_{1,t}\\
\vdots\\
\widetilde{\theta}_{a,t}\\
\vdots\\
\widetilde{\theta}_{K,t}
\end{pmatrix}
\] where $\bm{0}_d$ is the $d$-dimensional zero vector. 

Here, the loss is generated as $\ell_t(a, X_t) = \left\langle \theta_t,  \phi(a, X_t)\right\rangle  + \varepsilon_t(a)$. Note that $\theta$ is a  $(\widetilde{d}\times \widetilde{K})$ vector. 

\paragraph{P.Dep $(d, K, T)$} This setting is identical to our setting. Therefore, we omit the details of this setting.

As well as the case with P.Indep $(\widetilde{d}, \widetilde{K}, \widetilde{T})$, we can transform the P.Dep $(d, J, T)$ into the P.Indep $(dK, K, T)$. 

Let us define 
\[\widetilde{\phi}(X_t) = \begin{pmatrix}
    \phi(1, X_t)\\
    \vdots\\
    \phi(a, X_t)\\
    \vdots\\
    \phi(K, X_t)
\end{pmatrix},\quad 
\widetilde{\theta}_{a, t} = \begin{pmatrix} \bm{0}_d\\
\vdots\\
\theta_t\\
\vdots\\
\bm{0}_d
\end{pmatrix}
\] where $\bm{0}_d$ is the $d$-dimensional zero vector. 

Here, the loss is generated as $\widetilde{\ell}_t(a, X_t) = \left\langle \widetilde{\theta}_{a, t},  \widetilde{\phi}( X_t)\right\rangle + \widetilde{\varepsilon}_t(a)$. 

\paragraph{Regret transformation theorem} 
The above construction directly yields Theorem~\ref{thm:reg_trans} as follows.

\begin{proof}[Proof of Theorem~\ref{thm:reg_trans}]
    Suppose that there exists an algorithm whose regret is $R_T = f(d, K, T)$ in the P.Dep $(d, K, T)$, where $f:\mathbb{N}\times \mathbb{N}\times \mathbb{N}\to \mathbb{R}$.

    Consider the P.Indep $(\widetilde{d}, \widetilde{K}, \widetilde{T})$. For this problem, we can apply the algorithm by transforming the arm-independent features and parameters to the ones in the arm-dependent feature setting as 
    \[\phi(a, X_t) = \begin{pmatrix}
        \bm{0}_d\\
        \vdots\\
        \widetilde{\phi}(X_t)\\
        \vdots\\
        \bm{0}_d
    \end{pmatrix},\quad 
    \theta_t = \begin{pmatrix} \widetilde{\theta}_{1,t}\\
    \vdots\\
    \widetilde{\theta}_{a,t}\\
    \vdots\\
    \widetilde{\theta}_{K,t}
    \end{pmatrix}
    \]

    This problem has $\widetilde{d}\widetilde{K}$-dimensional parameter $\theta_t$. Therefore, the regret under this setting becomes $f(\widetilde{d}\widetilde{K}, \widetilde{K}, \widetilde{T})$.

    Similarly, we can prove the reverse case. Thus, the proof is complete.
\end{proof}

\end{document}